%% file: arxiv.tex
\documentclass[article]{elsarticle}

\usepackage{graphicx}
\usepackage[ruled,vlined]{algorithm2e}
\SetKwComment{Comment}{$\%$\ }{}
\usepackage{amsmath,amssymb,amsfonts}
\usepackage{dsfont}
\usepackage{tikz}
\usetikzlibrary{arrows}
\usepackage{xcolor}
\tikzset{
  treenode/.style = {align=center, inner sep=0pt, text centered,
    font=\sffamily},
  arn_n/.style = {treenode, circle, white, font=\sffamily\bfseries, draw=black,
    fill=black, text width=1.5em},
  arn_r/.style = {treenode, circle, red, draw=red, 
    text width=1.8em, very thick},
  arn_x/.style = {treenode, rectangle, draw=black,
    minimum width=0.5em, minimum height=0.5em}
}
\usepackage{svg}
\usepackage{subfig}
\graphicspath{{./img/}}
\usepackage{multirow}
\usepackage{url}
\usepackage{booktabs}
\definecolor{antiquebrass}{rgb}{0.8, 0.58, 0.46}
\usepackage[colorlinks,citecolor=blue]{hyperref}
\usepackage{mathtools}
\usepackage{floatrow}
\newfloatcommand{capbtabbox}{table}[][\FBwidth]
\usepackage[misc]{ifsym}

\newtheorem{theorem}{Theorem}
\newtheorem{example}{Example}
\newtheorem{proof}{Proof}

\begin{document}

\begin{frontmatter}

\title{Affinity-Based Hierarchical Learning of Dependent Concepts for Human Activity Recognition}

\author{
Aomar Osmani\textsuperscript{1}
and Massinissa Hamidi\textsuperscript{1}(\Letter)
and
Pegah Alizadeh\textsuperscript{2}
}
\address{
\textsuperscript{1} LIPN-UMR CNRS 7030, Univ. Sorbonne Paris Nord, Villetaneuse, France\\
\textsuperscript{2} Léonard de Vinci Pôle Universitaire, Research
Center, 92 916 Paris, La Défense\\
{\tt \{ao,hamidi\}@lipn.univ-paris13.fr}, {\tt pegah.alizadeh@devinci.fr}\\
}




\begin{abstract}
In multi-class classification tasks, like human activity recognition, it is often assumed that classes are separable. In real applications, this assumption becomes strong and generates inconsistencies. Besides, the most commonly used approach is to learn classes one-by-one against the others. This computational simplification principle introduces strong inductive biases on the learned theories. In fact, the natural connections among some classes, and not others, deserve to be taken into account. In this paper, we show that the organization of overlapping classes (multiple inheritances) into hierarchies considerably improves classification performances. This is particularly true in the case of activity recognition tasks featured in the SHL dataset.
After theoretically showing the exponential complexity of possible class hierarchies, we propose an approach based on transfer affinity among the classes to determine an optimal hierarchy for the learning process.
Extensive experiments show improved performances and a reduction in the number of examples needed to learn.
\end{abstract}
\begin{keyword}
Activity recognition \sep Dependent concepts \sep Meta-modeling.
\end{keyword}

\end{frontmatter}


\section{Introduction}
\label{sec:introduction}
Many real-world applications considered in machine learning exhibit dependencies among the various to-be-learned concepts (or classes)~\cite{essaidi2015learning,silla2011survey}.
This is particularly the case in human activity recognition from wearable sensor deployments which constitutes the main focus of our paper. This problem is two-folds: the high volume of accumulated data and the criteria selection optimization.  
For instance, are the criteria used to distinguish between the activities (concepts) {\it running} and {\it walking} the same as those used to distinguish between \textit{driving a car} and \textit{being in a bus}?
what about distinguishing each individual activity against the remaining ones taken as a whole?
Similarly, during the annotation process, when should someone consider that {\it walking} at a higher pace corresponds actually to {\it running}?
These questions naturally arise in the case of the SHL dataset~\cite{gjoreski2018university} which exhibits such dependencies. The considered activities in this dataset are difficult to separate due to the existence of many overlaps among certain activities. Some of the important causes for these overlaps are: (1) the on-body sensors deployments featured by this dataset, due to sensors coverage overlaps, tend to capture movements that are not necessarily related to a unique activity.
Authors in~\cite{hamidi2020data}, for example, have exhibited such overlaps; (2) The difficulty of data annotation during data collection conducted in real-world conditions.
For instance, the annotation issues can include the time-shift of a label with respect to the activity~\cite{stikic2009activity}, as well as wrong or missing labels~\cite{nguyen2014robust}.
Similarly, long lines of research in computer vision~\cite{taran2019impact} and time-series analysis~\cite{stikic2009activity,nguyen2014robust} raised these issues which hinder the development and large-scale adoption of these applications.

To solve these problems, we propose an original approach for structuring the considered concepts into hierarchies in a way that very similar concepts are grouped together and tackled by specialized classifiers.
The idea is that classifications at different levels of the hierarchy may rely on different features, or different combinations of the same features~\cite{zhou2011hierarchical}.
Indeed, many real-world classification problems are naturally cast as hierarchical classification problems~\cite{cai2004hierarchical,wehrmann2018hierarchical,yao2019hierarchically,zhou2011hierarchical}.
A work on the semantic relationships among the categories in a hierarchical structure shows that they are usually of the type \textit{generalization-specialization}~\cite{zhou2011hierarchical}. In other words, the lower-level categories are supposed to have the same general properties as the higher-level categories plus additional more specific properties.
The problem at hand is twice difficult as we have to, first, find the most appropriate hierarchical structure and, second, find optimal learners assigned to the nodes of the hierarchical structure.

We propose a data-driven approach to structure the considered concepts in a bottom-up approach. We start by computing the affinities and dependencies that exist among the concepts and fuse hierarchically the closest concepts with each other.
We leverage for this a powerful technique based on transfer which showed interesting empirical properties in various domains~\cite{zamir2018taskonomy,peters2019tune}.
Taking a bottom-up approach allows us to leverage learning the complete hierarchy (including the classifiers assigned to each non-leaf node) incrementally by reusing what was learned on the way.
Our contributions are as follows:
(1) we propose a theoretical calculation for computing the total number of tree hierarchical combinations (the search space for the optimal solution) based on the given number of concepts; 
(2) we propose an approach based on transfer affinity to determine an optimal organization of the concepts that improves both learning performances and accelerates the learning process;
(3) extensive experiments show the effectiveness of organizing the learning process. We noticeably get a substantial improvement of recognition performances over a baseline which uses a flat classification setting;
(4) we perform a comprehensive comparative analysis of the various stages of our approach which raises interesting questions about concept dependencies and the required amount of supervision.
\section{Problem Statement}\label{sec:problem-statement}
In this section, we briefly review the problem of hierarchical structuring of  the concepts in terms of formulation and background. We then provide a complexity analysis of the problem size and its search space.
\subsection{Problem Formulation and Background}\label{sec:problem-formulation-and-background}
Let $\mathcal{X} \subset \mathbb{R}^n$ be the inputs vector~\footnote{In our case, we select several body-motion modalities to be included in our experiments, among the 16 input modalities of the original dataset: \textit{accelerometer}, \textit{gyroscope}, etc. 
Segmentation and processing details are detailed in experimental part.} and let $\mathcal{C}$ be the set of atomic concepts (or labels) to learn.
The main idea of this paper comes from the fact that the concepts to be learned are not totally independent, thus grouping some concepts to learn them against the others using implicit biases considerably improves the quality of learning for each concept. 
The main problem is to find the best structure of concepts groups to be learned in order to optimize the learning of each atomic concept.
For this we follow the three dimensions setting defined in~\cite{kosmopoulos2015evaluation}, and we consider: (1) single-label classification as opposed to multi-label classification; (2) the type of hierarchy (or structure) to be trees as opposed to directed acyclic graphs; (3) instances that have to be classified into leafs, i.e. mandatory leaf node prediction~\cite{silla2011survey}, as opposed to the setting where instances can be classified into any node of the hierarchy (early stopping).

A tree hierarchy  organizes  the  class  labels  into  a  tree-like structure to represent a kind of "IS-A" relationship  between labels. Specifically,~\cite{kosmopoulos2015evaluation}  points out  that the  properties  of the "IS-A" relationship  can be described  as asymmetry, anti-reflexivity and transitivity~\cite{silla2011survey}.  We  define  a tree  as  a  pair $(\mathcal{C}, \prec)$, where $\mathcal{C}$ is the set of class labels and "$\prec$" denotes the "IS-A" relationship.

Let $\{(x_1, c_1), \dots, (x_N, c_N)\} \overset{\text{i.i.d.}}{\sim} X,C$ be a set of training examples, where $X$ and $C$ are two random variables taking values in $\mathcal{X} \times \mathcal{C}$, respectively. Each $x_k \in \mathcal{X}$ and each $c_k \in \mathcal{C}$.
Our goal is to learn a classification function $f: \mathcal{X} \xrightarrow{} \mathcal{C}$ that attains a small classification error.
In this paper, we associate each node $i$ with a classifier $\mathcal{M}_i$, and focus on classifiers $f(x)$ that are parameterized by $\mathcal{M}_1, \dots, \mathcal{M}_m$ through the following recursive procedure~\cite{zhou2011hierarchical} (check Fig.~\ref{fig:high-order-transfer}):
\begin{equation}
\small
  f(x)=\begin{cases}
    \textbf{initialize } i:=0\\
    \textbf{while } (Child(i) \text{ is not empty})
    \text{\;\;\;\;\;} i:= \operatorname*{argmax}_{j \in Child(i)} \mathcal{M}_j(x)\\
    \textbf{return } i \;\;\; \%\text{$Child(i)$ is the set of children for the node $i$}
  \end{cases}
\end{equation}
In the case of the SHL dataset, for instance, learning \textit{train} and \textit{subway} or \textit{car} and \textit{bus} before learning each concept alone gives better results.
As an advantage, considering these classes paired together as opposed to the flat classification setting leads to significant degradation of recognition performances as demonstrated in some works around the SHL dataset~\cite{wang2018summary}.
In contrast, organizing the various concepts into a tree-like structure, inspired by domain expertise, demonstrated significant gains in terms of recognition performances in the context of the SHL challenge~\cite{nakamura2018multi} and activity recognition in general~\cite{samie2020hierarchical,scheurer2020using}.

Designing such structures is of utmost importance but hard because it involves optimizing the structure as well as learning the weights of the classifiers attached to the nodes of that structure (see Sec.~\ref{sec:complexity-analysis}).
Our goal is then to determine an optimal structure of classes that can facilitate (improve and accelerate) learning of the whole concepts.

\subsection{Search Space Size: Complexity Analysis}
\label{sec:complexity-analysis}
A naive approach is to generate the lattice structure of concepts groups and to choose the tree hierarchies which give the best accuracy of atomic concepts. In practice, this is not doable because of the exponential (in the number of leaf nodes) number of possible trees.
We propose a recurrence relation involving binomial coefficients for calculating the total number of tree hierarchies for $K$ different concepts (class labels). 
\begin{example}
Assume we have $3$ various concepts, and we are interested in counting the total number of hierarchies for classifying these concepts. We consider that we have three classes namely $c_1, c_2$ and $c_3$, there exist $4$ different tree hierarchies for learning the classification problem as following: (1) ($c_1 c_2 c_3$) the tree has one level and the learning process takes one step. Three concepts are learned while each concept is learned separately from the others (flat classification), (2) (($c_1c_2$)$c_3$) the tree has two levels and the learning process takes two steps: at the first level, it learns two concepts (atomic $c_3$ and two atomics $c_1$ and $c_2$ together). At the second level it learns separately the two joined concepts $c_1$ and $c_2$ of the first level, etc and (3) ($c_1(c_2$ $c_3$)) and (4) (($c_1c_3$)$c_2$).
\end{example}

\begin{theorem}
Let $L(K)$ be the total number of trees for the given $K$ number of concepts. The total number of trees for $K+1$ concepts satisfies the following recurrence relation:
    $\label{equation:L-number}
    L(K+1) = {K \choose K-1} L(K)L(1) +
    2 \sum_{i=0}^{K-2} {K \choose i} L(i+1) L(K-i)$.
\label{theorem}
\end{theorem}

\begin{proof} 
It can be explained by observing that, for $K+1$ concepts containing $K$ existed concepts $c_1, \cdots c_K$ and a new added concept $\gamma$, we can produce the first level trees combinations as below. Notice that each atomic element $o$ can be one of the $c_1, \cdots c_K$ concepts. In order to compute the total number of trees combinations, we show what is the number of tree combinations by assigning the $K$ concepts to each item:
\begin{itemize}
    \item $(\gamma(\overbrace{o\cdots o}^{K \text{concepts}}))$: the number of trees combinations by taking the concept labels into the account are: ${K \choose 0} L(1) \times 2 \times L(K)$; the reason for multiplying the number of trees combinations for $K$ concepts to $2$ is because while the left side contains an atomic $\gamma$ concept, there are two choices for the right side of the tree in the first level: either we compute the total number of trees for $K$ concepts from the first level or we keep the first level as a $\overbrace{o\cdots o}^{K \text{concepts}}$ atomics and keep all $K$ concepts together, then continue the number of $K$ trees combinations from the second level of the tree. 
    \item $((\gamma o)(\overbrace{o\cdots o}^{K-1 \text{concepts}}))$: similar to the previous part we have ${K \choose 1} L(2) \times 2 \times L(K-1)$ trees combinations by taking the concepts labels into the account. ${K \choose 1}$ indicates the number of combinations for choosing a concept from the $K$ concept and put it with the new concept separately. While $L(2)$ is the number of trees combinations for the left side of tree separated with the new concept $\gamma$.
    \item $((\gamma oo)(\overbrace{o\cdots o}^{K-2 \text{concepts}}))$, $\cdots$
    \item $((\gamma \overbrace{o\cdots o}^{K-1 \text{concepts}})o)$: ${K \choose K-1} L(K) L(1)$ in this special part, we follow the same formula except the single concept in the right side has only one possible combination in the first level equal to $L(1)$.
\end{itemize}
All in all, the sum of these items calculates the total number of tree hierarchies for $K+1$ concepts.
\end{proof}
The first few number of total number of trees combinations for $1, 2, 3,4, 5, 6$, $7, 8, 9,10, \cdots$ concepts are: $1~, ~1, ~4, ~26, ~236, ~2752, ~39208, ~660032, ~12818912,~$ $282137824,\cdots $. 
In the case of the SHL dataset that we use in the empirical evaluation, we have $8$ different concepts and thus, the number of different types of hierarchies for this case is $L(8) = 660,032$.
\section{Proposed Approach}\label{sec:proposed-approach}
Our goals are to: (i) organize the considered concepts into hierarchies such that the learning process accounts for the dependencies existed among these concepts; (ii) characterize optimal classifiers that are associated to each non-leaf node of the hierarchies.
Structuring the concepts can be performed using two different approaches: a \textbf{top-down} approach where we seek to decompose the learning process; and a \textbf{bottom-up} approach where the specialized models are grouped together based on their affinities.
Our approach takes the latter direction and constructs hierarchies based on the similarities between concepts. This is because, an hierarchical approach as a bottom-up method is efficient in the case of high volume SHL data-sets.
In this section, we detail the different parts of our approach which are illustrated in Fig.~\ref{fig:proposed-hierarchical-approach}.
In the rest of this section, we introduce the three stages of our approach in detail: {\it Concept similarity analysis}, \textit{Hierarchy derivation}, and \textit{Hierarchy refinement}.
\begin{figure*}[h!]
\captionsetup[subfigure]{labelformat=empty}
\centering
\sffamily
\subfloat[]{
    \def\svgwidth{2.6\columnwidth}
    \resizebox{110mm}{!}{
       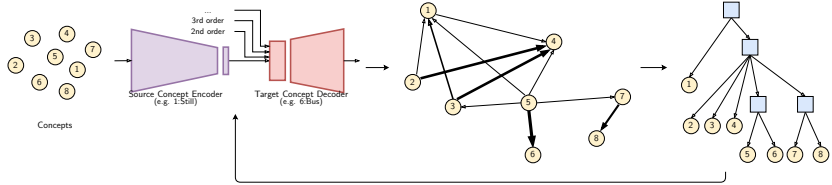
    }
}%
\caption{
    \small
    Our solution involves several repetitions of 3 main steps: (1) Concept similarity analysis:
    encoders are trained to output, for each source concept, an appropriate representation which is then fine-tuned to serve target concepts. Affinity scores are depicted by the arrows between concepts (the thicker the arrow, the higher the affinity score).
    (2) Hierarchy derivation:
    based on the obtained affinity scores, a hierarchy is derived using an agglomerate approach.
    (3) Hierarchy refinement: each non-leaf node of the derived hierarchy is assigned with a model that encompasses the  appropriate representation as well as an ERM which is optimized to separate the considered concepts.
}
\label{fig:proposed-hierarchical-approach}
\end{figure*}

\subsection{Concept Similarity (Affinity) Analysis}\label{sec:concept-affinity-analysis}
In our \textbf{bottom-up} approach we leverage transferability and dependency among concepts as a measure of similarity.
Besides the nice empirical properties of this measure (explained in the \textit{Properties} paragraph below), the argument behind it is to reuse what has been learned so far at the lower levels of the hierarchies.
Indeed, we leverage the models that we learned during this step and use them with few additional adjustments in the final hierarchical learning setting.
\paragraph{Transfer-based affinity}
Given the set of concepts $\mathcal{C}$, we compute during this step an affinity matrix that captures the notion of transferability and similarity, among the concepts.
For this, we first compute for each concept $c_i \in \mathcal{C}$ an encoder $f_{\theta}^{c_i}$ (parameterized by $\theta$) that learns to map the $c_i$ labeled inputs, to $\mathcal{Z}_{c_i}$.
Learning the encoder's parameters consists in minimizing the reconstruction error, satisfying the following optimization~\cite{vincent2010stacked}: $\operatorname*{argmin}_{\theta, \theta'} \mathbb{E}_{x, c \sim X, C | c=c_i} \mathcal{L}(g_{\theta'}^{c_i}(f_{\theta}^{c_i}(x)), x)$, where $g_{\theta'}^{c_i}$ is a decoder (parameterized by $\theta'$) which maps back the learned representation into the original inputs space.
We propose to leverage the learned encoder, for a given concept $c_i$, to compute affinities with other concepts via fine-tuning of the learned representation.
Precisely, we fine-tune the encoder $f_{c_i}^{\theta}$ to account for a target concept $c_j \in \mathcal{C}$. This process consists, similarly, in minimizing the reconstruction error, however rather than using the decoder $g_{\theta'}^{c_i}$ learned above, we design a genuine decoder $g_{\theta'}^{c_j}$ that we learn from the scratch.
The corresponding objective function is $\operatorname*{argmin}_{\theta, \theta'} \mathbb{E}_{x, c \sim X, C | c=c_j} \mathcal{L}(g_{\theta'}^{c_j}(f_{\theta}^{c_i}(x)), x)$.
We use the performance of this step as a \textit{similarity score} from $c_i$ to $c_j$ which we denote by $p_{c_i \xrightarrow{} c_j} \in [0, 1]$.
We refer to the number of examples belonging to the concept $c_j$ used during fine-tuning as the \textit{supervision budget}, denoted as $b$, which is used to index a given measure of similarity. It allows us to have an additional indicator as to the similarity between the considered concepts.
The final similarity score is computed as $\frac{\alpha \cdot p_{c_i \xrightarrow{} c_j} + \beta \cdot b}{\alpha + \beta}$.
We set $\alpha$ and $\beta$ to be equal to $\frac{1}{2}$.
\paragraph{Properties}
In many applications, e.g. computer-vision~\cite{zamir2018taskonomy} and natural language processing~\cite{peters2019tune}, several variants of the transfer-based similarity measure have been shown empirically to improve (i) the \textbf{quality} of transferred models (wins against fully supervised models), (ii) the \textbf{gains}, i.e. win rate against a network trained from scratch using the same training data as transfer networks', and more importantly (iii) the \textbf{universality} of the resulting structure. Indeed, the affinities based on transferability are stable despite the variations of a big corpus of hyperparameters.
We provide empirical evidence (Sec.~\ref{sec:evaluation:affinity-analysis}) of the appropriateness of the transfer-based affinity measure for the separability of the similar concepts and the difficulty to separate concepts that exhibit low similarity scores.
\subsection{Hierarchy Derivation}\label{sec:hierarchy-derivation}
Given the set of \textit{affinity scores} obtained previously, we derive the most appropriate hierarchy, following an agglomerative clustering method combined with some additional constraints.
The agglomerative clustering method proceeds by a series of successive fusions of the concepts into groups and results in a structure represented by a two-dimensional diagram known as a dendrogram.
It works by (1) forming groups of concepts that are close enough and (2) updating the affinity scores based on the newly formed groups.
This process is defined by the recurrence formula proposed by~\cite{lance1967general}. If defines a distance between a group of concepts ($k$) and a group formed by fusing $i$ and $j$ groups ($ij$) as $d_{k(ij)} = \alpha_i d_{ki} + \alpha_j d_{kj} + \beta d_{ij} + \gamma |d_{ki} - d_{kj}|$, where $d_{ij}$ is the distance between two groups $i$ and $j$.
By varying the parameter values $\alpha_i, \alpha_j, \beta$, and $\gamma$, we expect to get clustering schemes with various characteristics.

In addition to the above updating process, we propose additional constraints to refine further the hierarchy derivation stage.
Given the dendrogram produced by the agglomerative method above, we define an \textit{affinity threshold} $\tau$ such that if the distance at a given node is $d_{ij} \geq \tau$, then we merge the nodes to form a unique subtree.
In addition, as we keep track of the quantities of data used to train and fine-tune the encoders during the transfer-based affinity analysis stage, this indicator is exploited to inform us as to which nodes to merge.
Let $\mathcal{T}$ be the derived hierarchy (tree) and let $t$ indexes the non-leaf or internal nodes.
The leafs of the hierarchy correspond to the considered concepts.
For any non-leaf node $t$, we associate a model $\mathcal{M}_t$ that encompasses (1) an encoder (denoted in the following simply by $\mathcal{Z}_t$ in order to focus on the representation) that maps inputs $X$ to representations $\mathcal{Z}_t$ and (2) an ERM (Empirical Risk Minimizer)~\cite{vapnik1992principles} $f_t$ (such as support vector machines SVMs) that outputs decision boundaries based on the representations produced by the encoder.

\subsection{Hierarchy Refinement}\label{sec:hierarchy-refinement}
After explaining the hierarchy derivation process, we will discuss: (1) which representations are used in each individual model; and (2) how each individual model (including the representation and the ERM weights) is adjusted to account for both the local errors and also those of the hierarchy as a whole.

\paragraph{Which representations to use?}
The question discussed here is related to the encoders to be used in each non-leaf node.
For any non-leaf node $t$ we distinguish two cases: (i) all its children are leafs; (ii) it has at least one non-leaf node.
In the first case, the final considered ERM representation, associated with the non-leaf node, is the representation learned in the concept affinity analysis step (first-order transfer-based affinity).
In the second case, we can either fuse the nodes (for example, in a case of classification between 3 concepts, we get all 3 together rather than, first \{1\} vs. \{2,3\}, then \{2\} vs. \{3\}), or keep them as they are and leverage the affinities based on higher-order transfer where, rather than accounting for a unique target concept, the representation is then fine-tuned.
Fig.~\ref{fig:high-order-transfer} illustrates how transfers are performed between non-leaf nodes models.
We index the models with the encoder $\mathcal{M}_{[\mathcal{Z}_i]}$. In the case of higher-order transfer, the models are indexed using all concepts involved in the transfer, i.e. $\mathcal{M}_{[\mathcal{Z}_{i,j, \dots}]}$.

\begin{figure}[h!]
\captionsetup[figure]{font=small}
\centering
\sffamily
\subfloat[]{
    \def\svgwidth{.25\columnwidth}
    \resizebox{27mm}{!}{
       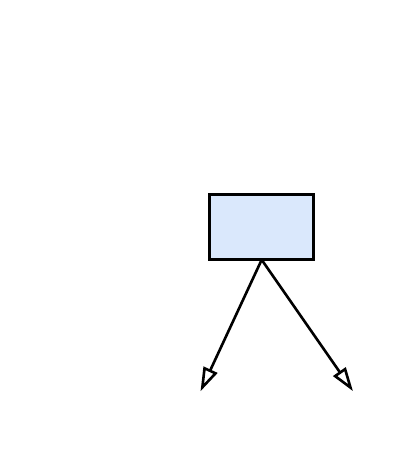
    }
    \label{fig:high-order-transfer}
}%
\subfloat[]{
    \def\svgwidth{.2\columnwidth}
    \resizebox{26mm}{!}{
       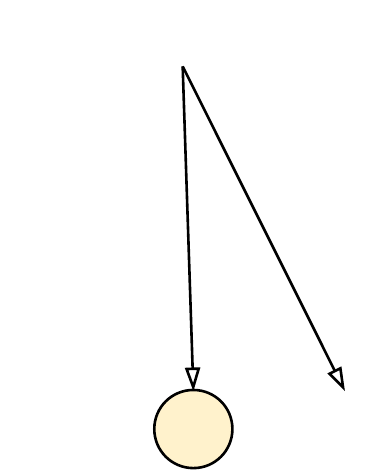
    }
    \label{fig:high-order-transfer-fusion}
}%
\subfloat[]{
    \def\svgwidth{0.3\columnwidth}
    \resizebox{30mm}{!}{
       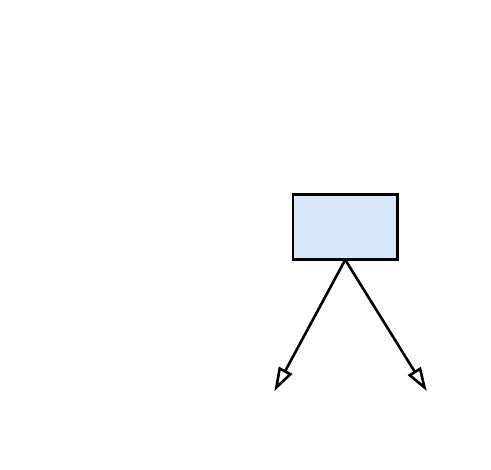
    }
    \label{fig:high-order-transfer-non-leaf}
}%
\caption{
\small
Transfers are performed between non-leaf nodes models.
The hierarchy in (a) can be kept as they are merged to form the hierarchy in (b). (b): a high-order transfer between the concepts $c_i, c_j$, and $c_k$ is performed. (c): no transfers can be made.
}
\label{fig:high-order-transfer}
\end{figure}

\paragraph{Adjusting models weights}
Classifiers are trained to output a hypothesis based on the most appropriate representations learned earlier.
Given the encoder (representation) assigned to any non-leaf node $t$, we select a classifier $\hat{f} := \operatorname*{argmin}_{f\in\mathcal{H}}$ $\hat{R}(f, \mathcal{Z}_t)$ where $\hat{R}(f, \mathcal{Z}_t) := \frac{1}{M} \sum_{x,c\sim X,C|c\in Child(t)} \mathbb{E}_{z \sim \mathcal{Z}_t|x} [\mathcal{L}(c, f(z))]$ and $\mathcal{H}$ is the hypothesis space.
Models are adjusted to account for local errors as well as for global errors related to the hierarchy as a whole. In the first case, the loss is defined as the traditional hinge loss used in SVMs which is intended to adjust the weights of the classifiers that have only children leaves.
In the second case, we use a loss that encourages the models to leverage orthogonal representations (between children and parent nodes)~\cite{zhou2011hierarchical}.

\section{Experiments and Results}
Empirical evaluation of our approach are performed on three steps: we evaluate classification performances in the hierarchical setting (Sec.~\ref{sec:evaluation:hierarchical-classification-performances}); then, we evaluate the transfer-based affinity analysis step and the properties related to the separability of the considered concepts (Sec.~\ref{sec:evaluation:affinity-analysis});
finally, we evaluate the derived hierarchies in terms of stability, performance, and agreement with their counterparts defined by domain experts (Sec.~\ref{sec:evaluation:hierarchy-derivation})~\footnote{Software package and code to reproduce empirical results are publicly available at \url{https://github.com/sensor-rich/hierarchicalSHL}}.
Training details can be found in~\ref{sec:training-details} and evaluation metrics are detailed in~\ref{sec:evaluation-metrics}.

\paragraph{SHL dataset \cite{gjoreski2018university}}
\label{sec:dataset-description}
It is a highly versatile and precisely annotated dataset dedicated to mobility-related human activity recognition.
In contrast to related representative datasets like~\cite{carpineti2018custom}, the SHL dataset (26.43 GB) provides , simultaneously, multimodal and multilocation locomotion data recorded in real-life settings.
Among the 16 modalities of the original dataset, we select the body-motion modalities including: \textit{accelerometer}, \textit{gyroscope}, \textit{magnetometer}, \textit{linear acceleration}, \textit{orientation}, \textit{gravity}, and \textit{ambient pressure}.
This makes the data set suitable for a wide range of applications and in particular transportation recognition concerned with this paper. From the 8 primary categories of transportation, we are selected: {\it 1:Still}, {\it 2:Walk}, {\it 3:Run}, {\it 4:Bike}, {\it 5:Car}, {\it 6:Bus}, {\it 7:Train}, and {\it 8:Subway (Tube)}.

\section{Training Details}\label{sec:training-details}
We use Tensorflow for building the encoders/decoders. We construct encoders by stacking Conv1d/ReLU/MaxPool blocks.
These blocks are followed by a Fully Connected/ReLU layers.
Encoders performance estimation is based on the validation loss and is framed as a sequence classification problem.
As a preprocessing step, annotated input streams from the huge SHL dataset are segmented into sequences of 6000 samples which correspond to a duration of 1 min. given a sampling rate of 100 Hz.
For weight optimization, we use stochastic gradient descent  with  Nesterov  momentum  of  0.9  and a learning-rate of 0.1 for a minimum of 12 epochs (we stop training if there is no improvement).
Weight decay is set to 0.0001.
Furthermore, to make the neural networks more stable, we use batch normalization on top of each convolutional layer. We use SVMs as our ERMs in the derived hierarchies.

\section{Evaluation Metrics}\label{sec:evaluation-metrics}
In hierarchical classification settings, the hierarchical structure is important and should be taken into account during model evaluation~\cite{silla2011survey}.
Various measures that account for the hierarchical structure of the learning process have been studied in the literature.
They can be categorized into: distance-based; depth-dependent; semantics-based; and hierarchy-based measures. Each one is displaying advantages and disadvantages depending on the characteristics of the considered structure~\cite{costa2007review}.
In our experiments, we use the \textit{H-loss}, a hierarchy-based measure defined in~\cite{cesa2006incremental}. This measure captures the intuition that \textit{"whenever a classification mistake is made on a node of the taxonomy, then no loss should be charged for any additional mistake occurring in the sub-tree of that node."}
$\ell_H(\hat{y}, y) = \sum_{i=1}^{N} \{ \hat{y}_i \neq y_i \land \hat{y}_j = y_j, j \in Anc(i) \}$,
where $\hat{y} = (\hat{y}_1, \cdots \hat{y}_N)$ is the predicted labels, $y = (y_1, \cdots y_N)$ is the true labels, and $Anc(i)$ is the set of ancestors for the node $i$.

\subsection{Evaluation of the Hierarchical Classification Performances}\label{sec:evaluation:hierarchical-classification-performances}
In these experiments, we evaluate the flat classification setting using neural networks which constitute our baseline for the rest of the empirical evaluations.
To compare our baseline with the hierarchical models, we make sure to get the same complexity, i.e. comparable number of parameters as the largest hierarchies including the weights of the encoders and those of the ERMs.
We also use Bayesian optimization based on Gaussian processes as surrogate models to select the optimal hyperparameters of the baseline model~\cite{snoek2012practical,hamidi2020}.
More details about the baseline and its hyperparameters are available in the code repository~\cite{hamidi2020}.

\paragraph{Per-node performances}
Fig.~\ref{fig:per-node-accuracy-improvement} shows the resulting per-node performances, i.e. how accurately the models associated with the non-leaf nodes can predict the correct subcategory averaged over the entire derived hierarchies.
The nodes are ranked according to the obtained per-node performance (top 10 nodes are shown) and accompanied by their appearance frequency.
It is worth noticing that the concept 1:{\it still} learned alone against the rest of the concepts (first bar) achieves the highest gains in terms of recognition performances while the appearance frequency of this learning configuration is high (more than 60 times).
We see also that the concepts 4:{\it bike}, 5:{\it car}, and 6:{\it bus} grouped together (5th bar) occur very often in the derived hierarchies (80 times) which is accompanied by fairly significant performance gains ($5.09 \pm 0.3$\%).
At the same time, as expected, we see that the appearance frequency gets into a plateau starting from the 6th bar (which lasts after the 10th bar).
This suggests that the most influential nodes are often exhibited by our approach.
\begin{figure}
    \centering
    \sffamily
    \subfloat[]{
    \def\svgwidth{1.0\columnwidth}
    \resizebox{75mm}{!}{
       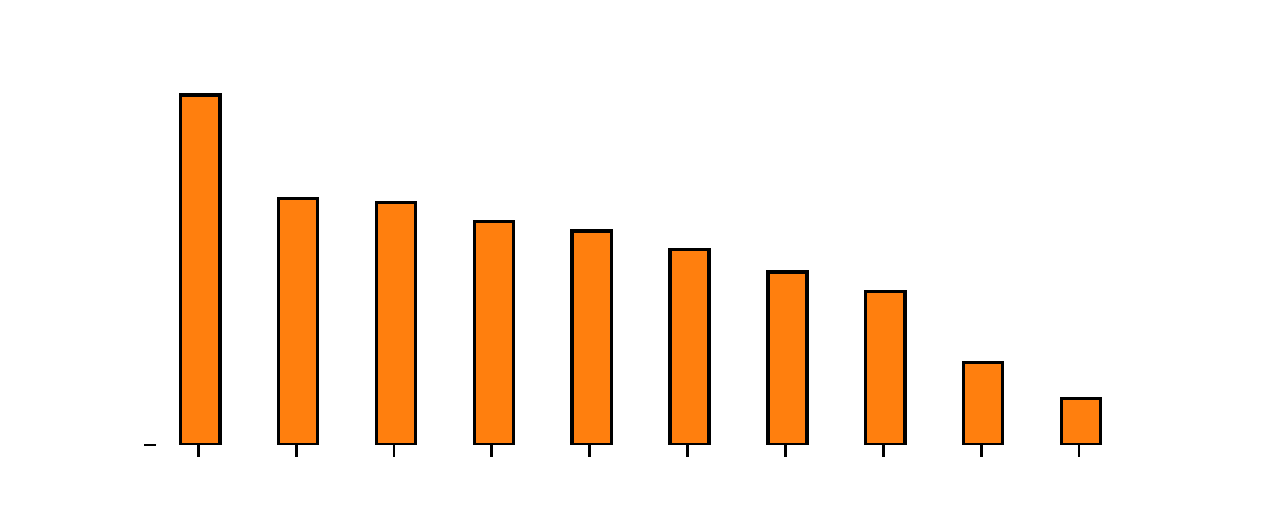
    }
    }
    
    \subfloat[]{
    \def\svgwidth{.99\columnwidth}
    \resizebox{75mm}{!}{
       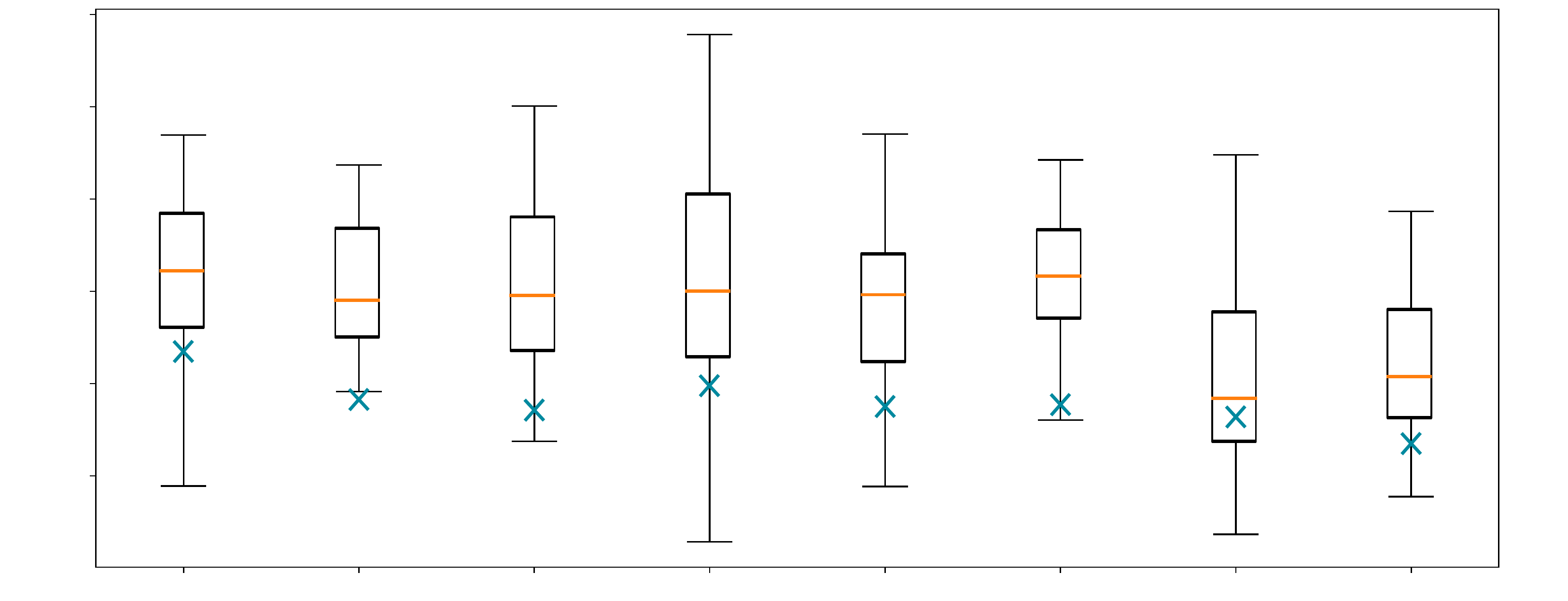
    }
    }
    \caption{
        (a) Per-node performance gains, averaged over the entire derived architectures (similar nodes are grouped and their performances are averaged).
        The appearance frequency of the nodes is also illustrated.
        Each bar represents the gained accuracy of each node in our hierarchical approach. For example, the 8th bar corresponds to the concepts 2:{\it walk}-3:{\it run}-4:{\it bike} grouped together.
        (b) Recognition performances of each individual concept, averaged over the entire derived hierarchies.
        For reference, the recognition performances of the baseline model are also shown.
    }
    \label{fig:per-node-accuracy-improvement}
\end{figure}
\paragraph{Per-concept performances}
We further ensure that the performance improvements we get at the node levels are reflected at the concept level. Experimental results show the recognition performances of each concept, averaged over the whole hierarchies derived using our proposed approach.
We indeed observe that there are significant improvements for each individual concept over the baseline (flat classification setting).
We observe that again 1:{\it still} has the highest classification rate ($72.32 \pm 3.45$\%) and an improvement of 5 points over the baseline.
Concept 6:{\it bus} also exhibits a roughly similar trend.
On the other hand, concept 7:{\it train} has the least gains ($64.43 \pm 4.45$\%) with no significant improvement over the baseline.
Concept 8:{\it subway} exhibits the same behavior suggesting that there are undesirable effects that stem from the definition of these two concepts.

\subsection{Evaluation of the Affinity Analysis Stage}
\label{sec:evaluation:affinity-analysis}
These experiments evaluate the proposed transfer-based affinity measure. We assess, the separability of the concepts depending on their similarity score (for both the transfer-affinity and supervision budget) and the learned representation.

\paragraph{Appropriateness of the Transfer-based Affinity Measure}
We reviewed above the nice properties of the transfer-based measure especially the universality and stability of the resulting affinity structure. The question that arises is related to the separability of the concepts that are grouped together. Are the obtained representations, are optimal for the final ERMs used for the classification?
This is what we investigate here.
Fig.~\ref{fig:high-affinity-score-decision-boundary} shows the decision boundaries generated by the considered ERMs which are provided with the learned representations of two concepts. The first case (top right), exhibits a low-affinity score, and the second case (bottom right) shows a high-affinity score.
In the first case, the boundaries are unable to separate the two concepts while it gets a fairly distinct frontier.
\begin{figure}
\centering
\sffamily
\subfloat[]{
    \centering
    \def\svgwidth{0.3\columnwidth}
    \resizebox{28mm}{!}{
       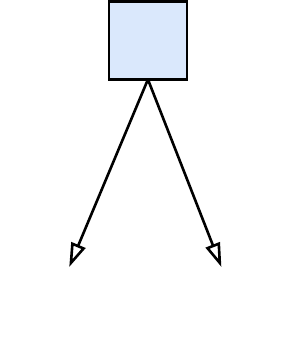
    }
}%
\subfloat[]{
    \def\svgwidth{0.2\columnwidth}
    \resizebox{38mm}{!}{
       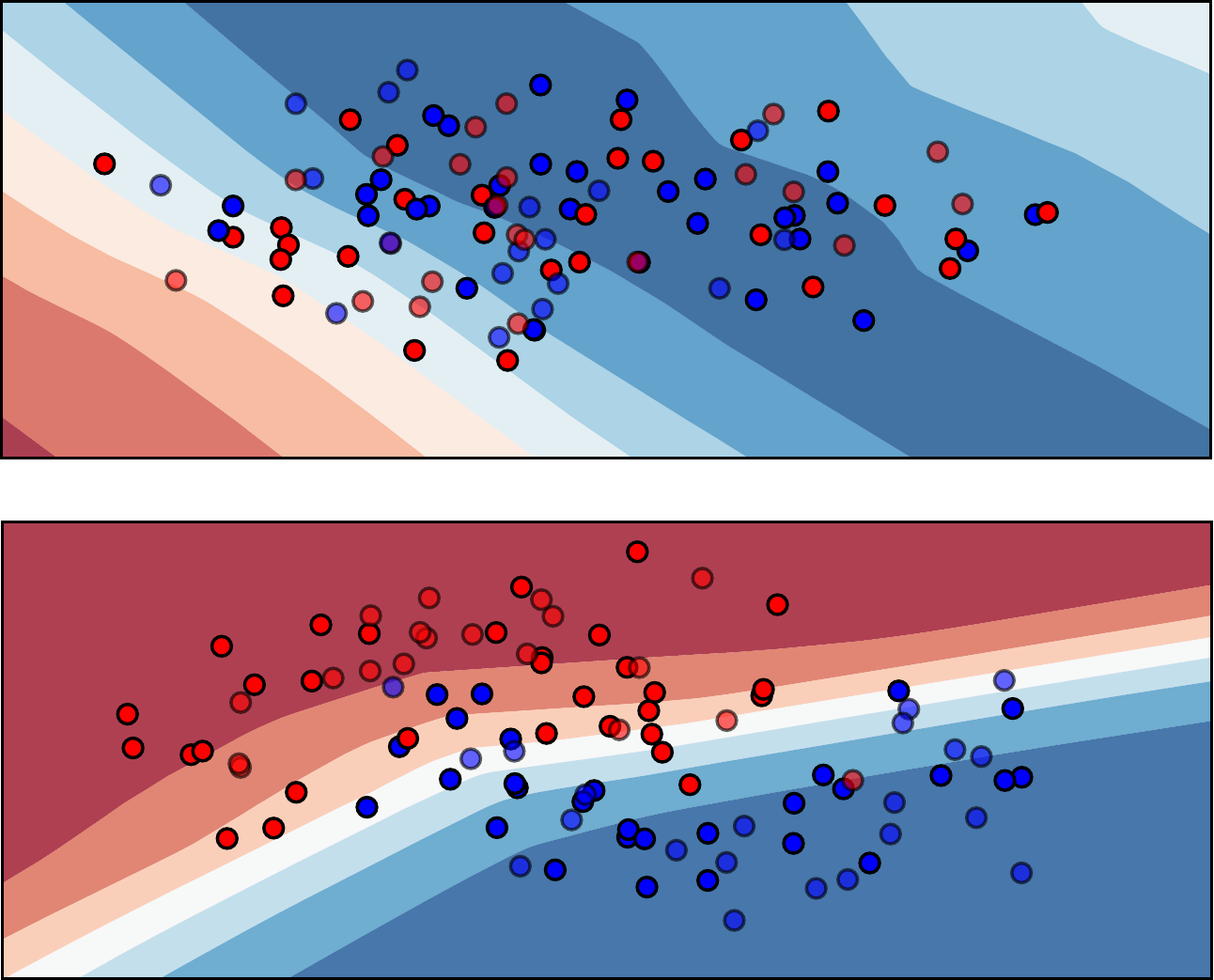
    }
    \label{fig:high-affinity-score-decision-boundary}
}%

\subfloat[]{
    \centering
    \def\svgwidth{1.\columnwidth}
    \resizebox{80mm}{!}{
       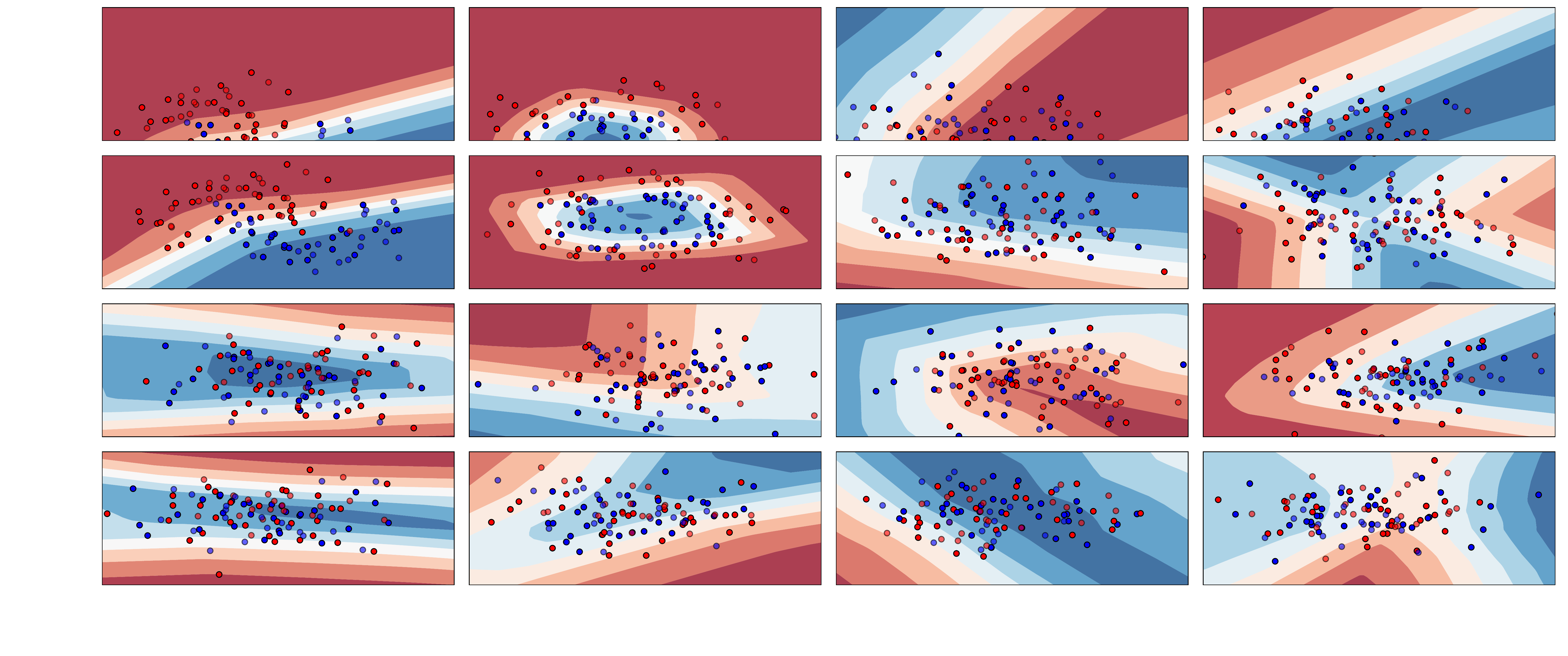
    }
     \label{fig:decision-boundaries}
}
\caption{
(a) Non-leaf node grouping concepts $c_i$ and $c_j$. (b) Decision boundaries generated by the ERM of the non-leaf node using an encoder (representation) fine-tuned to account for (top) the case where $c_i$ and $c_j$ are dissimilar (low-affinity score) and (bottom) the case where $c_i$ and $c_j$ are similar (high-affinity score).
(c) Decision boundaries obtained by SVM-based classifiers trained on the representations $\mathcal{Z}_t$ as a function of the distance between the concepts (y-axis) and the supervision budget (x-axis).
}
\label{fig:appropriateness-of-transfer-based-affinity-score-as-a-measure}
\end{figure}

\paragraph{Impact on the ERMs' Decision Boundaries}
We train different models with various learned representations in order to investigate the effect of the initial affinities (obtained solely with a set of 100 learning examples) and the supervision budget (additional learning examples used to fine-tune the obtained representation) on the classification performances of the ERMs associated with the non-leaf nodes of our hierarchies.
Fig.~\ref{fig:decision-boundaries} shows the decision boundaries generated by various models as a function of the distance between the concepts (y-axis) and the supervision budget (x-axis).
Increasing the supervision budget to some larger extents (more than $\sim 300$ examples) results in a substantial decrease in classification performances of the ERMs.
This suggests that, although our initial affinity scores are decisive (e.g. 0.8), the supervision budget is tightly linked to generalization. This shows that a trade-off (controlled by the supervision budget) between separability and initial affinities arises when we seek to group concepts together.
In other words, the important question is whether to increase the supervision budget indefinitely (in the limits of available learning examples) in order to find the most appropriate concepts to fuse with, while expecting good separability.

\subsection{Universality and Stability}
\label{sec:evaluation:hierarchy-derivation}
We demonstrated in the previous section the appropriateness of the transfer-based affinity measure to provide distance between concepts as well as the existence of a trade-off between concepts separability and their initial affinities.
Here we evaluate the \textbf{universality} of the derived hierarchies as well as their \textbf{stability} during adaptation with respect to our hyperparameters (affinity threshold and supervision budget).
We compare the derived hierarchies with their domain experts-defined counterparts, as well as those obtained via a random sampling process.
Fig.~\ref{fig:hierarchies} shows some of the hierarchies defined by the domain experts (first row) and sampled using the random sampling process.
For example, the hierarchy depicted in Fig.~\ref{fig:hierarchy-4} corresponds to a split between static (1:{\it still}, 5:{\it car}, 6:{\it bus}, 7:{\it train}, 8:{\it subway}) and dynamic (2:{\it walk}, 3:{\it run}, 4:{\it bike}) activities.
The difference between the hierarchies depicted in Fig.~\ref{fig:hierarchy-1} and~\ref{fig:hierarchy-2} is related to 4:{\it bike} activity which is linked first to 2:{\it walk} and 3:{\it run} then to 5:{\it car} and 6:{\it bus}.
A possible interpretation is that in the first case, biking is considered as ``on feet" activity while in the second case as ``on wheels" activity.
What we observed is that the derived hierarchies tend to converge towards the expert-defined ones.
\begin{figure}[h!]
\centering
\sffamily
\subfloat[]{
    \def\svgwidth{1.1\columnwidth}
    \resizebox{25mm}{!}{
       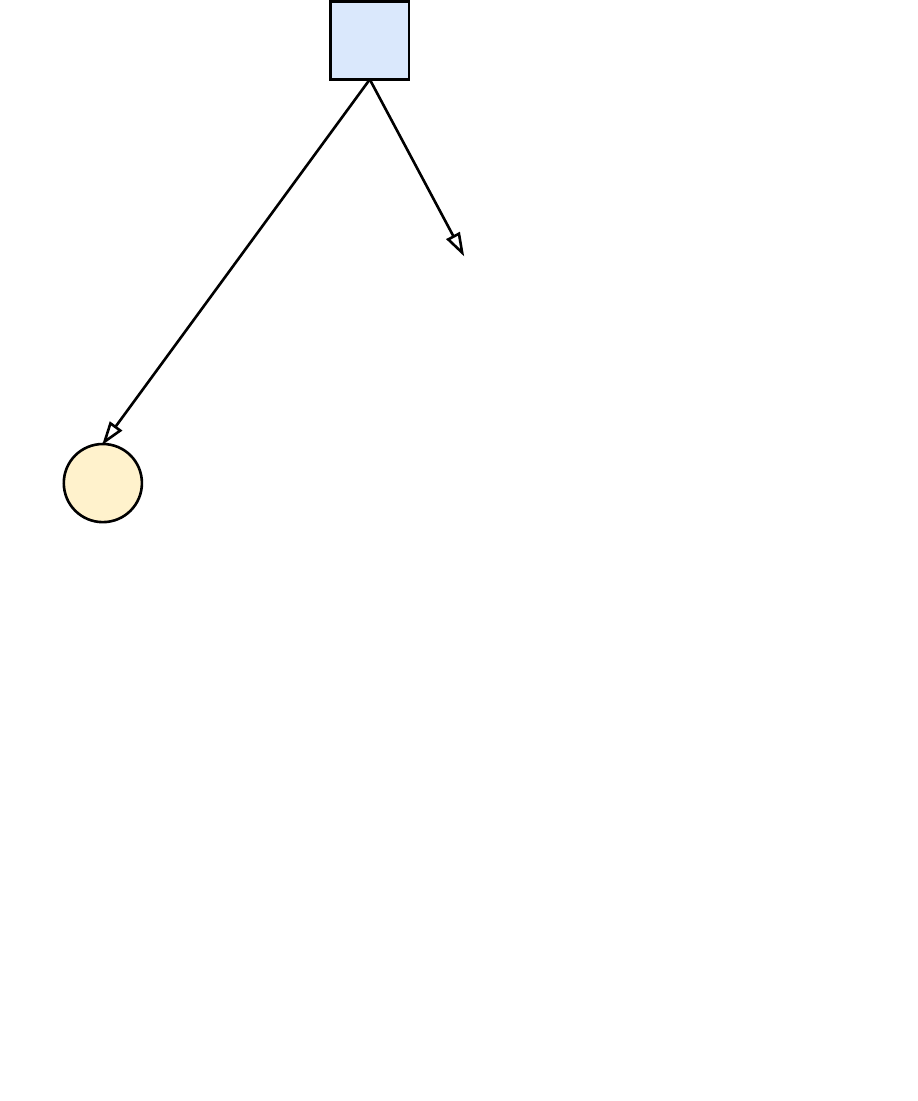
    }
    \label{fig:hierarchy-1}
}%
\subfloat[]{
    \def\svgwidth{1.1\columnwidth}
    \resizebox{25mm}{!}{
       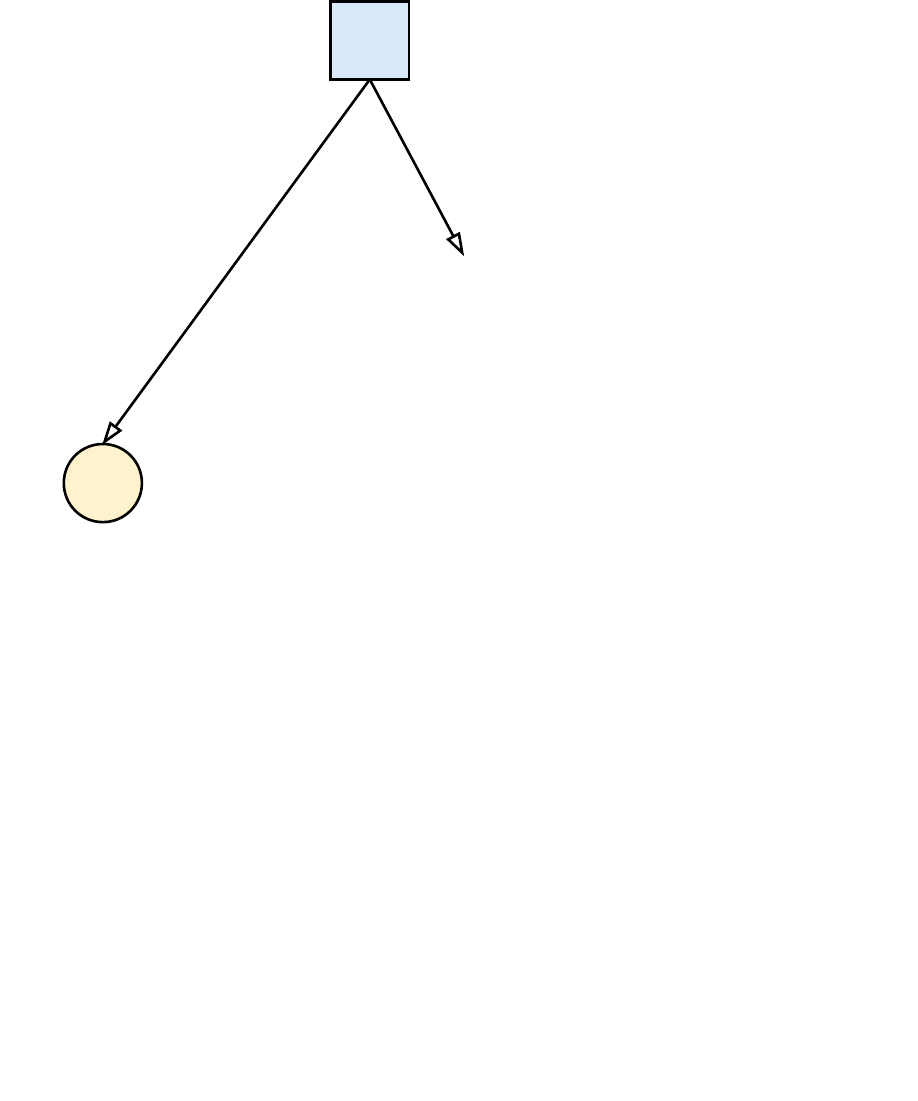
    }
    \label{fig:hierarchy-2}
}%
\subfloat[]{
    \def\svgwidth{1.1\columnwidth}
    \resizebox{25mm}{!}{
       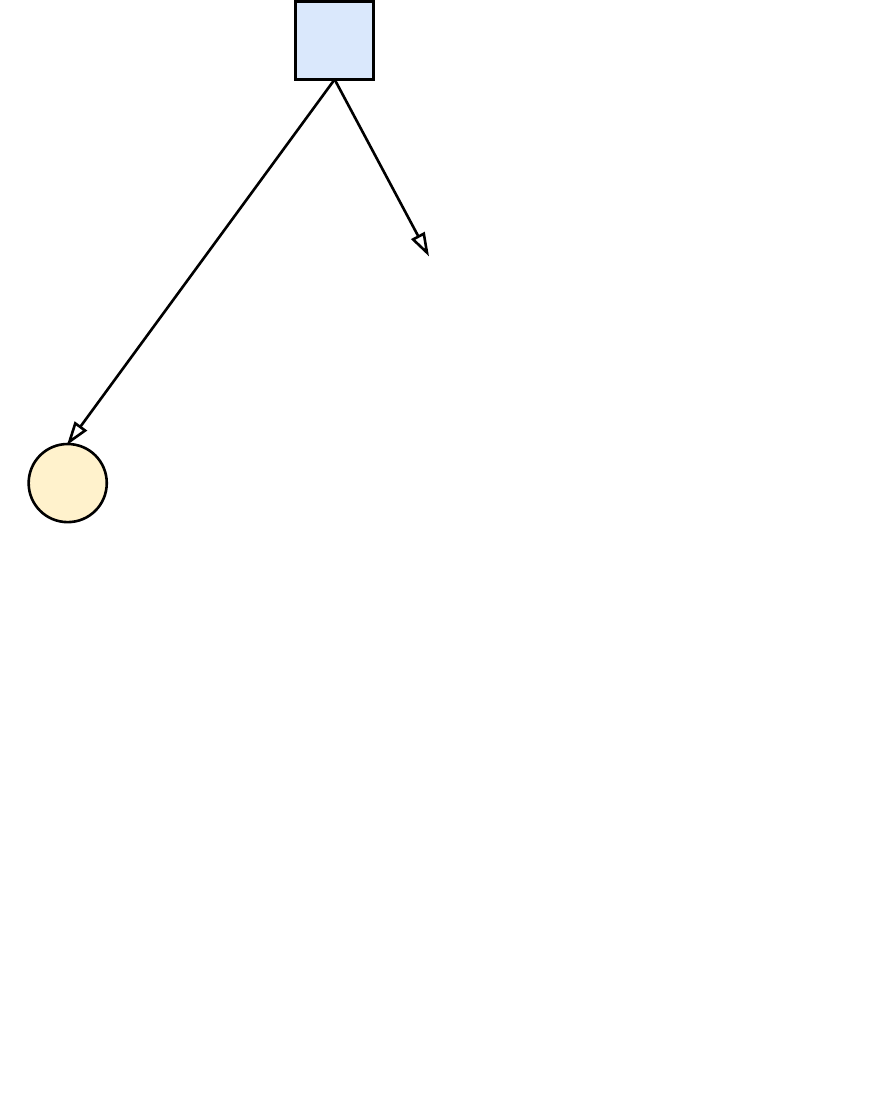
    }
    \label{fig:hierarchy-3}
}%
\subfloat[]{
    \def\svgwidth{1.1\columnwidth}
    \resizebox{25mm}{!}{
       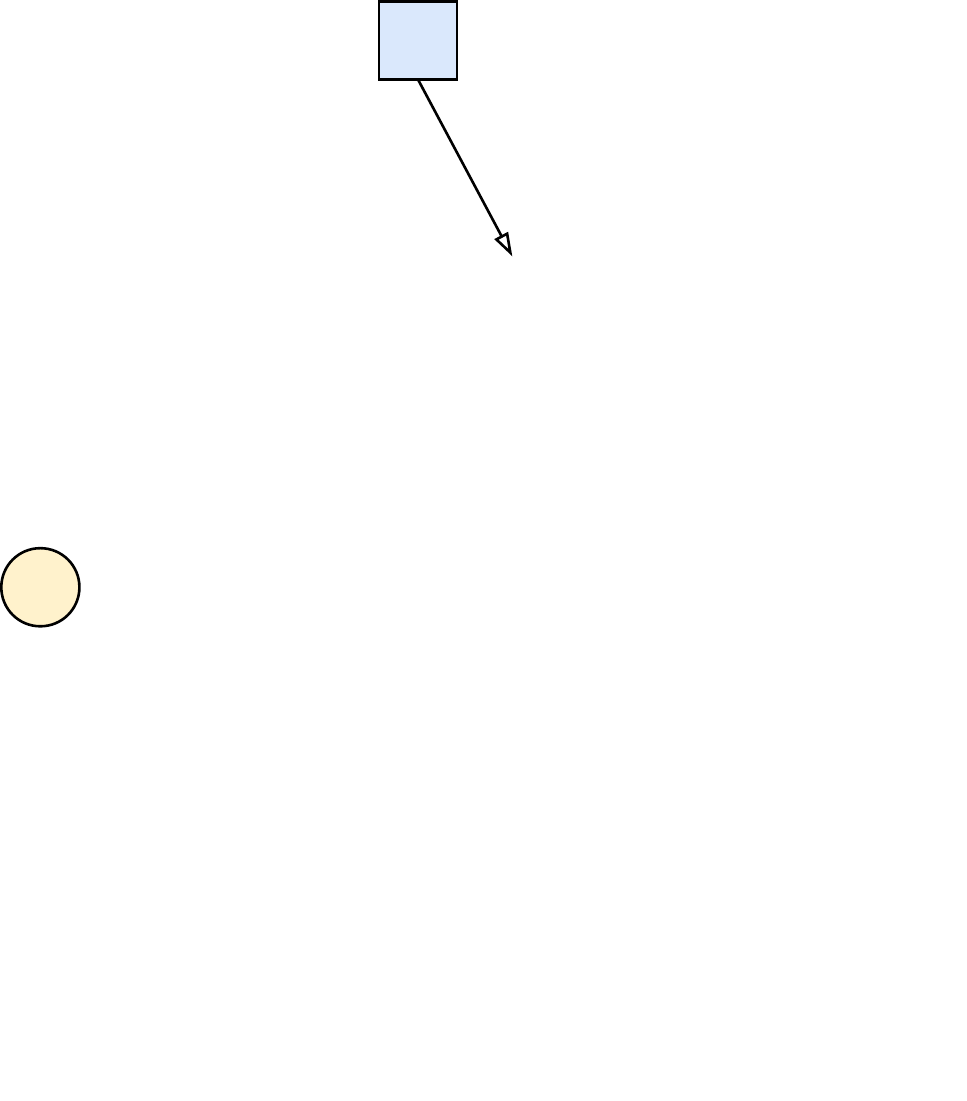
    }
    \label{fig:hierarchy-4}
}%
\caption{
Examples of hierarchies:
(a) defined via domain expertise,
(b-c) derived using our approach, and
(d) randomly sampled.
Concepts 1---8 from left to right.
}
\label{fig:hierarchies}
\end{figure}

\begin{table}[h!]
    \begin{tabular}{lcc}
    \toprule
        Method & Agree. & perf. avg.$\pm$ std. \\
    \midrule
        Expertise   & -     & 72.32$\pm$0.17 \\
        Random      & 0.32  & 48.17$\pm$5.76 \\
        Proposed    & 0.77  & 75.92$\pm$1.13 \\
    \bottomrule
    \end{tabular}
    \caption{
        Summary of the recognition performances obtained with our proposed approach compared to randomly sampled and expert-defined hierarchies.
    }
    \label{tab:my_label}
\end{table}

We compare the derived hierarchies in terms of their level of agreement. We use for this assessment, the Cohen's kappa coefficient~\cite{cohen1960coefficient} which measures the agreement between two raters.
The first column of Table~\ref{tab:my_label} provides the obtained coefficients.
We also compare the average recognition performance of the derived hierarchies (second column of Table~\ref{tab:my_label}).
In terms of stability, as we vary the design choices (hyperparameters), defined in our approach, we found that the affinity threshold has a substantial impact on our results with many adjustments involved (12 hierarchy adjustments on avg.) whereas the supervision budget has a slight effect, which confirms the observations in Sec.~\ref{sec:evaluation:affinity-analysis}.

\section{Conclusion and Future Work}\label{sec:conclusion-and-future-work}
This paper proposes an approach for organizing the learning process of dependent concepts in the case of human activity recognition.
We first determine a suitable structure for the concepts according to a transfer affinity-based measure.
We then characterize optimal representations and classifiers which are then refined to account for both local and global errors.
We provide theoretical bounds for the problem and empirically show that using our approach we are able to improve the performances and robustness of activity recognition models over a flat classification baseline.
In addition to supporting the necessity of organizing concepts learning, our experiments raise interesting questions for future work.
Noticeably, Sec.~\ref{sec:evaluation:affinity-analysis} asks what is the optimal amount of supervision for deriving the hierarchies. Another future work is to study different approaches for searching and exploring the search space of different hierarchical types (lattices,~etc.).


\end{document}

%% file: img/proposed-hierarchical-approach_1.pdf_tex
\begingroup%
  \makeatletter%
  \providecommand\color[2][]{%
    \errmessage{(Inkscape) Color is used for the text in Inkscape, but the package 'color.sty' is not loaded}%
    \renewcommand\color[2][]{}%
  }%
  \providecommand\transparent[1]{%
    \errmessage{(Inkscape) Transparency is used (non-zero) for the text in Inkscape, but the package 'transparent.sty' is not loaded}%
    \renewcommand\transparent[1]{}%
  }%
  \providecommand\rotatebox[2]{#2}%
  \newcommand*\fsize{\dimexpr\f@size pt\relax}%
  \newcommand*\lineheight[1]{\fontsize{\fsize}{#1\fsize}\selectfont}%
  \ifx\svgwidth\undefined%
    \setlength{\unitlength}{1181.25bp}%
    \ifx\svgscale\undefined%
      \relax%
    \else%
      \setlength{\unitlength}{\unitlength * \real{\svgscale}}%
    \fi%
  \else%
    \setlength{\unitlength}{\svgwidth}%
  \fi%
  \global\let\svgwidth\undefined%
  \global\let\svgscale\undefined%
  \makeatother%
  \begin{picture}(1,0.22)%
    \lineheight{1}%
    \setlength\tabcolsep{0pt}%
    \put(0,0){\includegraphics[width=\unitlength,page=1]{proposed-hierarchical-approach_1.pdf}}%
    \put(0.82952381,0.11587302){\makebox(0,0)[t]{\lineheight{1.25}\smash{\begin{tabular}[t]{c}1\end{tabular}}}}%
    \put(0,0){\includegraphics[width=\unitlength,page=2]{proposed-hierarchical-approach_1.pdf}}%
    \put(0.83174603,0.06603175){\makebox(0,0)[t]{\lineheight{1.25}\smash{\begin{tabular}[t]{c}2\end{tabular}}}}%
    \put(0,0){\includegraphics[width=\unitlength,page=3]{proposed-hierarchical-approach_1.pdf}}%
    \put(0.85777778,0.06603175){\makebox(0,0)[t]{\lineheight{1.25}\smash{\begin{tabular}[t]{c}3\end{tabular}}}}%
    \put(0,0){\includegraphics[width=\unitlength,page=4]{proposed-hierarchical-approach_1.pdf}}%
    \put(0.88444444,0.06603175){\makebox(0,0)[t]{\lineheight{1.25}\smash{\begin{tabular}[t]{c}4\end{tabular}}}}%
    \put(0,0){\includegraphics[width=\unitlength,page=5]{proposed-hierarchical-approach_1.pdf}}%
    \put(0.90063492,0.03111111){\makebox(0,0)[t]{\lineheight{1.25}\smash{\begin{tabular}[t]{c}5\end{tabular}}}}%
    \put(0,0){\includegraphics[width=\unitlength,page=6]{proposed-hierarchical-approach_1.pdf}}%
    \put(0.93396825,0.03111111){\makebox(0,0)[t]{\lineheight{1.25}\smash{\begin{tabular}[t]{c}6\end{tabular}}}}%
    \put(0,0){\includegraphics[width=\unitlength,page=7]{proposed-hierarchical-approach_1.pdf}}%
    \put(0.95809524,0.03111111){\makebox(0,0)[t]{\lineheight{1.25}\smash{\begin{tabular}[t]{c}7\end{tabular}}}}%
    \put(0,0){\includegraphics[width=\unitlength,page=8]{proposed-hierarchical-approach_1.pdf}}%
    \put(0.98984127,0.03111111){\makebox(0,0)[t]{\lineheight{1.25}\smash{\begin{tabular}[t]{c}8\end{tabular}}}}%
    \put(0,0){\includegraphics[width=\unitlength,page=9]{proposed-hierarchical-approach_1.pdf}}%
    \put(0.51174603,0.20650794){\makebox(0,0)[t]{\lineheight{1.25}\smash{\begin{tabular}[t]{c}1\end{tabular}}}}%
    \put(0,0){\includegraphics[width=\unitlength,page=10]{proposed-hierarchical-approach_1.pdf}}%
    \put(0.49269841,0.11920635){\makebox(0,0)[t]{\lineheight{1.25}\smash{\begin{tabular}[t]{c}2\end{tabular}}}}%
    \put(0,0){\includegraphics[width=\unitlength,page=11]{proposed-hierarchical-approach_1.pdf}}%
    \put(0.54190476,0.08904762){\makebox(0,0)[t]{\lineheight{1.25}\smash{\begin{tabular}[t]{c}3\end{tabular}}}}%
    \put(0,0){\includegraphics[width=\unitlength,page=12]{proposed-hierarchical-approach_1.pdf}}%
    \put(0.6647619,0.1684127){\makebox(0,0)[t]{\lineheight{1.25}\smash{\begin{tabular}[t]{c}4\end{tabular}}}}%
    \put(0,0){\includegraphics[width=\unitlength,page=13]{proposed-hierarchical-approach_1.pdf}}%
    \put(0.63412698,0.09412698){\makebox(0,0)[t]{\lineheight{1.25}\smash{\begin{tabular}[t]{c}5\end{tabular}}}}%
    \put(0,0){\includegraphics[width=\unitlength,page=14]{proposed-hierarchical-approach_1.pdf}}%
    \put(0.63920635,0.03063492){\makebox(0,0)[t]{\lineheight{1.25}\smash{\begin{tabular}[t]{c}6\end{tabular}}}}%
    \put(0,0){\includegraphics[width=\unitlength,page=15]{proposed-hierarchical-approach_1.pdf}}%
    \put(0.7484127,0.10142857){\makebox(0,0)[t]{\lineheight{1.25}\smash{\begin{tabular}[t]{c}7\end{tabular}}}}%
    \put(0,0){\includegraphics[width=\unitlength,page=16]{proposed-hierarchical-approach_1.pdf}}%
    \put(0.71666667,0.05063492){\makebox(0,0)[t]{\lineheight{1.25}\smash{\begin{tabular}[t]{c}8\end{tabular}}}}%
    \put(0,0){\includegraphics[width=\unitlength,page=17]{proposed-hierarchical-approach_1.pdf}}%
    \put(0.1036254,0.15888889){\makebox(0,0)[t]{\lineheight{1.25}\smash{\begin{tabular}[t]{c}7\end{tabular}}}}%
    \put(0,0){\includegraphics[width=\unitlength,page=18]{proposed-hierarchical-approach_1.pdf}}%
    \put(0.07187936,0.10809524){\makebox(0,0)[t]{\lineheight{1.25}\smash{\begin{tabular}[t]{c}8\end{tabular}}}}%
    \put(0,0){\includegraphics[width=\unitlength,page=19]{proposed-hierarchical-approach_1.pdf}}%
    \put(0.03854603,0.11920635){\makebox(0,0)[t]{\lineheight{1.25}\smash{\begin{tabular}[t]{c}6\end{tabular}}}}%
    \put(0,0){\includegraphics[width=\unitlength,page=20]{proposed-hierarchical-approach_1.pdf}}%
    \put(0.05759365,0.15){\makebox(0,0)[t]{\lineheight{1.25}\smash{\begin{tabular}[t]{c}5\end{tabular}}}}%
    \put(0,0){\includegraphics[width=\unitlength,page=21]{proposed-hierarchical-approach_1.pdf}}%
    \put(0.07187936,0.17793651){\makebox(0,0)[t]{\lineheight{1.25}\smash{\begin{tabular}[t]{c}4\end{tabular}}}}%
    \put(0,0){\includegraphics[width=\unitlength,page=22]{proposed-hierarchical-approach_1.pdf}}%
    \put(0.02970159,0.17269841){\makebox(0,0)[t]{\lineheight{1.25}\smash{\begin{tabular}[t]{c}3\end{tabular}}}}%
    \put(0,0){\includegraphics[width=\unitlength,page=23]{proposed-hierarchical-approach_1.pdf}}%
    \put(0.00952381,0.13984127){\makebox(0,0)[t]{\lineheight{1.25}\smash{\begin{tabular}[t]{c}2\end{tabular}}}}%
    \put(0,0){\includegraphics[width=\unitlength,page=24]{proposed-hierarchical-approach_1.pdf}}%
    \put(0.08457778,0.13444444){\makebox(0,0)[t]{\lineheight{1.25}\smash{\begin{tabular}[t]{c}1\end{tabular}}}}%
    \put(0,0){\includegraphics[width=\unitlength,page=25]{proposed-hierarchical-approach_1.pdf}}%
    \put(0.24394285,0.19428571){\makebox(0,0)[t]{\lineheight{1.25}\smash{\begin{tabular}[t]{c}{\small 3rd order}\end{tabular}}}}%
    \put(0,0){\includegraphics[width=\unitlength,page=26]{proposed-hierarchical-approach_1.pdf}}%
    \put(0.24394285,0.18111111){\makebox(0,0)[t]{\lineheight{1.25}\smash{\begin{tabular}[t]{c}{\small 2nd order}\end{tabular}}}}%
    \put(0,0){\includegraphics[width=\unitlength,page=27]{proposed-hierarchical-approach_1.pdf}}%
    \put(0.24394285,0.20698413){\makebox(0,0)[t]{\lineheight{1.25}\smash{\begin{tabular}[t]{c}{\small ...}\end{tabular}}}}%
    \put(0,0){\includegraphics[width=\unitlength,page=28]{proposed-hierarchical-approach_1.pdf}}%
    \put(0.20457777,0.10396825){\makebox(0,0)[t]{\lineheight{1.25}\smash{\begin{tabular}[t]{c}Source Concept Encoder\end{tabular}}}}%
    \put(0.20457777,0.09507937){\makebox(0,0)[t]{\lineheight{1.25}\smash{\begin{tabular}[t]{c}{\small (e.g. 1:Still)}\end{tabular}}}}%
    \put(0,0){\includegraphics[width=\unitlength,page=29]{proposed-hierarchical-approach_1.pdf}}%
    \put(0.35759366,0.10396825){\makebox(0,0)[t]{\lineheight{1.25}\smash{\begin{tabular}[t]{c}Target Concept Decoder\end{tabular}}}}%
    \put(0.35759366,0.09507937){\makebox(0,0)[t]{\lineheight{1.25}\smash{\begin{tabular}[t]{c}{\small (e.g. 6:Bus)}\end{tabular}}}}%
    \put(0,0){\includegraphics[width=\unitlength,page=30]{proposed-hierarchical-approach_1.pdf}}%
    \put(0.05759365,0.06650794){\makebox(0,0)[t]{\lineheight{1.25}\smash{\begin{tabular}[t]{c}Concepts\end{tabular}}}}%
    \put(0,0){\includegraphics[width=\unitlength,page=31]{proposed-hierarchical-approach_1.pdf}}%
  \end{picture}%
\endgroup%

%% file: img/high-order-transfer.pdf_tex
\begingroup%
  \makeatletter%
  \providecommand\color[2][]{%
    \errmessage{(Inkscape) Color is used for the text in Inkscape, but the package 'color.sty' is not loaded}%
    \renewcommand\color[2][]{}%
  }%
  \providecommand\transparent[1]{%
    \errmessage{(Inkscape) Transparency is used (non-zero) for the text in Inkscape, but the package 'transparent.sty' is not loaded}%
    \renewcommand\transparent[1]{}%
  }%
  \providecommand\rotatebox[2]{#2}%
  \newcommand*\fsize{\dimexpr\f@size pt\relax}%
  \newcommand*\lineheight[1]{\fontsize{\fsize}{#1\fsize}\selectfont}%
  \ifx\svgwidth\undefined%
    \setlength{\unitlength}{117.16110938bp}%
    \ifx\svgscale\undefined%
      \relax%
    \else%
      \setlength{\unitlength}{\unitlength * \real{\svgscale}}%
    \fi%
  \else%
    \setlength{\unitlength}{\svgwidth}%
  \fi%
  \global\let\svgwidth\undefined%
  \global\let\svgscale\undefined%
  \makeatother%
  \begin{picture}(1,1.15385985)%
    \lineheight{1}%
    \setlength\tabcolsep{0pt}%
    \put(0,0){\includegraphics[width=\unitlength,page=1]{high-order-transfer.pdf}}%
    \put(0.64003654,0.56652758){\makebox(0,0)[t]{\lineheight{1.25}\smash{\begin{tabular}[t]{c}$\mathcal{M}_{[z_j]}$\end{tabular}}}}%
    \put(0,0){\includegraphics[width=\unitlength,page=2]{high-order-transfer.pdf}}%
    \put(0.49120302,0.07041586){\makebox(0,0)[t]{\lineheight{1.25}\smash{\begin{tabular}[t]{c}$c_j$\end{tabular}}}}%
    \put(0,0){\includegraphics[width=\unitlength,page=3]{high-order-transfer.pdf}}%
    \put(0.86248663,0.07041586){\makebox(0,0)[t]{\lineheight{1.25}\smash{\begin{tabular}[t]{c}$c_k$\end{tabular}}}}%
    \put(0,0){\includegraphics[width=\unitlength,page=4]{high-order-transfer.pdf}}%
    \put(0.41598608,1.04183462){\makebox(0,0)[t]{\lineheight{1.25}\smash{\begin{tabular}[t]{c}$\mathcal{M}_{[z_{i, j}]}$\end{tabular}}}}%
    \put(0,0){\includegraphics[width=\unitlength,page=5]{high-order-transfer.pdf}}%
    \put(0.1727313,0.45450235){\makebox(0,0)[t]{\lineheight{1.25}\smash{\begin{tabular}[t]{c}$c_i$\end{tabular}}}}%
    \put(0,0){\includegraphics[width=\unitlength,page=6]{high-order-transfer.pdf}}%
  \end{picture}%
\endgroup%

%% file: img/high-order-transfer-fusion.pdf_tex
\begingroup%
  \makeatletter%
  \providecommand\color[2][]{%
    \errmessage{(Inkscape) Color is used for the text in Inkscape, but the package 'color.sty' is not loaded}%
    \renewcommand\color[2][]{}%
  }%
  \providecommand\transparent[1]{%
    \errmessage{(Inkscape) Transparency is used (non-zero) for the text in Inkscape, but the package 'transparent.sty' is not loaded}%
    \renewcommand\transparent[1]{}%
  }%
  \providecommand\rotatebox[2]{#2}%
  \newcommand*\fsize{\dimexpr\f@size pt\relax}%
  \newcommand*\lineheight[1]{\fontsize{\fsize}{#1\fsize}\selectfont}%
  \ifx\svgwidth\undefined%
    \setlength{\unitlength}{110.81043636bp}%
    \ifx\svgscale\undefined%
      \relax%
    \else%
      \setlength{\unitlength}{\unitlength * \real{\svgscale}}%
    \fi%
  \else%
    \setlength{\unitlength}{\svgwidth}%
  \fi%
  \global\let\svgwidth\undefined%
  \global\let\svgscale\undefined%
  \makeatother%
  \begin{picture}(1,1.21998888)%
    \lineheight{1}%
    \setlength\tabcolsep{0pt}%
    \put(0,0){\includegraphics[width=\unitlength,page=1]{high-order-transfer-fusion.pdf}}%
    \put(0.49914465,0.07445147){\makebox(0,0)[t]{\lineheight{1.25}\smash{\begin{tabular}[t]{c}$c_j$\end{tabular}}}}%
    \put(0,0){\includegraphics[width=\unitlength,page=2]{high-order-transfer-fusion.pdf}}%
    \put(0.89170695,0.07445147){\makebox(0,0)[t]{\lineheight{1.25}\smash{\begin{tabular}[t]{c}$c_k$\end{tabular}}}}%
    \put(0,0){\includegraphics[width=\unitlength,page=3]{high-order-transfer-fusion.pdf}}%
    \put(0.47139457,1.10154336){\makebox(0,0)[t]{\lineheight{1.25}\smash{\begin{tabular}[t]{c}$\mathcal{M}_{[z_{i, j,k}]}$\end{tabular}}}}%
    \put(0,0){\includegraphics[width=\unitlength,page=4]{high-order-transfer-fusion.pdf}}%
    \put(0.16072887,0.4805504){\makebox(0,0)[t]{\lineheight{1.25}\smash{\begin{tabular}[t]{c}$c_i$\end{tabular}}}}%
    \put(0,0){\includegraphics[width=\unitlength,page=5]{high-order-transfer-fusion.pdf}}%
  \end{picture}%
\endgroup%

%% file: img/high-order-transfer-non-leaf.pdf_tex
\begingroup%
  \makeatletter%
  \providecommand\color[2][]{%
    \errmessage{(Inkscape) Color is used for the text in Inkscape, but the package 'color.sty' is not loaded}%
    \renewcommand\color[2][]{}%
  }%
  \providecommand\transparent[1]{%
    \errmessage{(Inkscape) Transparency is used (non-zero) for the text in Inkscape, but the package 'transparent.sty' is not loaded}%
    \renewcommand\transparent[1]{}%
  }%
  \providecommand\rotatebox[2]{#2}%
  \newcommand*\fsize{\dimexpr\f@size pt\relax}%
  \newcommand*\lineheight[1]{\fontsize{\fsize}{#1\fsize}\selectfont}%
  \ifx\svgwidth\undefined%
    \setlength{\unitlength}{141.21970885bp}%
    \ifx\svgscale\undefined%
      \relax%
    \else%
      \setlength{\unitlength}{\unitlength * \real{\svgscale}}%
    \fi%
  \else%
    \setlength{\unitlength}{\svgwidth}%
  \fi%
  \global\let\svgwidth\undefined%
  \global\let\svgscale\undefined%
  \makeatother%
  \begin{picture}(1,0.95728494)%
    \lineheight{1}%
    \setlength\tabcolsep{0pt}%
    \put(0,0){\includegraphics[width=\unitlength,page=1]{high-order-transfer-non-leaf.pdf}}%
    \put(0.70136096,0.4700123){\makebox(0,0)[t]{\lineheight{1.25}\smash{\begin{tabular}[t]{c}$\mathcal{M}_{[z_j]}$\end{tabular}}}}%
    \put(0,0){\includegraphics[width=\unitlength,page=2]{high-order-transfer-non-leaf.pdf}}%
    \put(0.55817977,0.05841961){\makebox(0,0)[t]{\lineheight{1.25}\smash{\begin{tabular}[t]{c}$c_j$\end{tabular}}}}%
    \put(0,0){\includegraphics[width=\unitlength,page=3]{high-order-transfer-non-leaf.pdf}}%
    \put(0.86621043,0.05841961){\makebox(0,0)[t]{\lineheight{1.25}\smash{\begin{tabular}[t]{c}$c_k$\end{tabular}}}}%
    \put(0,0){\includegraphics[width=\unitlength,page=4]{high-order-transfer-non-leaf.pdf}}%
    \put(0.50432752,0.86434465){\makebox(0,0)[t]{\lineheight{1.25}\smash{\begin{tabular}[t]{c}$\mathcal{M}_{[z_{i, j}]}$\end{tabular}}}}%
    \put(0,0){\includegraphics[width=\unitlength,page=5]{high-order-transfer-non-leaf.pdf}}%
    \put(0.29906225,0.4700123){\makebox(0,0)[t]{\lineheight{1.25}\smash{\begin{tabular}[t]{c}$\mathcal{M}_{[z_j]}$\end{tabular}}}}%
    \put(0,0){\includegraphics[width=\unitlength,page=6]{high-order-transfer-non-leaf.pdf}}%
    \put(0.19932406,0.3088273){\makebox(0,0)[t]{\lineheight{1.25}\smash{\begin{tabular}[t]{c}...\end{tabular}}}}%
  \end{picture}%
\endgroup%

%% file: img/per-node-accuracy-improvement.pdf_tex
\begingroup%
  \makeatletter%
  \providecommand\color[2][]{%
    \errmessage{(Inkscape) Color is used for the text in Inkscape, but the package 'color.sty' is not loaded}%
    \renewcommand\color[2][]{}%
  }%
  \providecommand\transparent[1]{%
    \errmessage{(Inkscape) Transparency is used (non-zero) for the text in Inkscape, but the package 'transparent.sty' is not loaded}%
    \renewcommand\transparent[1]{}%
  }%
  \providecommand\rotatebox[2]{#2}%
  \newcommand*\fsize{\dimexpr\f@size pt\relax}%
  \newcommand*\lineheight[1]{\fontsize{\fsize}{#1\fsize}\selectfont}%
  \ifx\svgwidth\undefined%
    \setlength{\unitlength}{362.96405849bp}%
    \ifx\svgscale\undefined%
      \relax%
    \else%
      \setlength{\unitlength}{\unitlength * \real{\svgscale}}%
    \fi%
  \else%
    \setlength{\unitlength}{\svgwidth}%
  \fi%
  \global\let\svgwidth\undefined%
  \global\let\svgscale\undefined%
  \makeatother%
  \begin{picture}(1,0.41878136)%
    \lineheight{1}%
    \setlength\tabcolsep{0pt}%
    \put(0,0){\includegraphics[width=\unitlength,page=1]{per-node-accuracy-improvement.pdf}}%
    \put(0.08611282,0.0558264){\makebox(0,0)[lt]{\lineheight{1.25}\smash{\begin{tabular}[t]{l}0\end{tabular}}}}%
    \put(0.08611282,0.12194864){\makebox(0,0)[lt]{\lineheight{1.25}\smash{\begin{tabular}[t]{l}2\end{tabular}}}}%
    \put(0.08611282,0.18807089){\makebox(0,0)[lt]{\lineheight{1.25}\smash{\begin{tabular}[t]{l}4\end{tabular}}}}%
    \put(0.08611282,0.25419314){\makebox(0,0)[lt]{\lineheight{1.25}\smash{\begin{tabular}[t]{l}6\end{tabular}}}}%
    \put(0.08611282,0.32444803){\makebox(0,0)[lt]{\lineheight{1.25}\smash{\begin{tabular}[t]{l}8\end{tabular}}}}%
    \put(0,0){\includegraphics[width=\unitlength,page=2]{per-node-accuracy-improvement.pdf}}%
    \put(0.90437565,0.0558264){\makebox(0,0)[lt]{\lineheight{1.25}\smash{\begin{tabular}[t]{l}0\end{tabular}}}}%
    \put(0.90437565,0.12194864){\makebox(0,0)[lt]{\lineheight{1.25}\smash{\begin{tabular}[t]{l}20\end{tabular}}}}%
    \put(0.90437565,0.18807089){\makebox(0,0)[lt]{\lineheight{1.25}\smash{\begin{tabular}[t]{l}40\end{tabular}}}}%
    \put(0.90437565,0.25419314){\makebox(0,0)[lt]{\lineheight{1.25}\smash{\begin{tabular}[t]{l}60\end{tabular}}}}%
    \put(0.90437565,0.32444803){\makebox(0,0)[lt]{\lineheight{1.25}\smash{\begin{tabular}[t]{l}80\end{tabular}}}}%
    \put(0.14810243,0.03103055){\makebox(0,0)[lt]{\lineheight{1.25}\smash{\begin{tabular}[t]{l}1\end{tabular}}}}%
    \put(0.2266226,0.03103055){\makebox(0,0)[lt]{\lineheight{1.25}\smash{\begin{tabular}[t]{l}2\end{tabular}}}}%
    \put(0.30514277,0.03103055){\makebox(0,0)[lt]{\lineheight{1.25}\smash{\begin{tabular}[t]{l}3\end{tabular}}}}%
    \put(0.3795303,0.03103055){\makebox(0,0)[lt]{\lineheight{1.25}\smash{\begin{tabular}[t]{l}4\end{tabular}}}}%
    \put(0.45805047,0.03103055){\makebox(0,0)[lt]{\lineheight{1.25}\smash{\begin{tabular}[t]{l}5\end{tabular}}}}%
    \put(0.53657064,0.03103055){\makebox(0,0)[lt]{\lineheight{1.25}\smash{\begin{tabular}[t]{l}6\end{tabular}}}}%
    \put(0.61509081,0.03103055){\makebox(0,0)[lt]{\lineheight{1.25}\smash{\begin{tabular}[t]{l}7\end{tabular}}}}%
    \put(0.69361099,0.03103055){\makebox(0,0)[lt]{\lineheight{1.25}\smash{\begin{tabular}[t]{l}8\end{tabular}}}}%
    \put(0.76799852,0.03103055){\makebox(0,0)[lt]{\lineheight{1.25}\smash{\begin{tabular}[t]{l}9\end{tabular}}}}%
    \put(0.83825341,0.03103055){\makebox(0,0)[lt]{\lineheight{1.25}\smash{\begin{tabular}[t]{l}10\end{tabular}}}}%
    \put(0,0){\includegraphics[width=\unitlength,page=3]{per-node-accuracy-improvement.pdf}}%
    \put(0.73367407,0.34155914){\makebox(0,0)[lt]{\lineheight{1.25}\smash{\begin{tabular}[t]{l}Perf. gains\end{tabular}}}}%
    \put(0.73367407,0.3134881){\makebox(0,0)[lt]{\lineheight{1.25}\smash{\begin{tabular}[t]{l}Appear. freq.\end{tabular}}}}%
    \put(0,0){\includegraphics[width=\unitlength,page=4]{per-node-accuracy-improvement.pdf}}%
    \put(0.04917459,0.21571029){\rotatebox{90}{\makebox(0,0)[t]{\lineheight{1.25}\smash{\begin{tabular}[t]{c}Performance gains (\%)\end{tabular}}}}}%
    \put(0.51132664,0.00220858){\rotatebox{-0.143022}{\makebox(0,0)[t]{\lineheight{1.25}\smash{\begin{tabular}[t]{c}Nodes\end{tabular}}}}}%
    \put(0.97901872,0.21571029){\rotatebox{90}{\makebox(0,0)[t]{\lineheight{1.25}\smash{\begin{tabular}[t]{c}Appearance frequency\end{tabular}}}}}%
  \end{picture}%
\endgroup%

%% file: img/per-concept-accuracy-improvement.pdf_tex
\begingroup%
  \makeatletter%
  \providecommand\color[2][]{%
    \errmessage{(Inkscape) Color is used for the text in Inkscape, but the package 'color.sty' is not loaded}%
    \renewcommand\color[2][]{}%
  }%
  \providecommand\transparent[1]{%
    \errmessage{(Inkscape) Transparency is used (non-zero) for the text in Inkscape, but the package 'transparent.sty' is not loaded}%
    \renewcommand\transparent[1]{}%
  }%
  \providecommand\rotatebox[2]{#2}%
  \newcommand*\fsize{\dimexpr\f@size pt\relax}%
  \newcommand*\lineheight[1]{\fontsize{\fsize}{#1\fsize}\selectfont}%
  \ifx\svgwidth\undefined%
    \setlength{\unitlength}{935.13769056bp}%
    \ifx\svgscale\undefined%
      \relax%
    \else%
      \setlength{\unitlength}{\unitlength * \real{\svgscale}}%
    \fi%
  \else%
    \setlength{\unitlength}{\svgwidth}%
  \fi%
  \global\let\svgwidth\undefined%
  \global\let\svgscale\undefined%
  \makeatother%
  \begin{picture}(1,0.38304886)%
    \lineheight{1}%
    \setlength\tabcolsep{0pt}%
    \put(0,0){\includegraphics[width=\unitlength,page=1]{per-concept-accuracy-improvement.pdf}}%
    \put(0.01567373,0.19868071){\rotatebox{90}{\makebox(0,0)[t]{\lineheight{1.25}\smash{\begin{tabular}[t]{c}Classification performances (\%)\end{tabular}}}}}%
    \put(0,0){\includegraphics[width=\unitlength,page=2]{per-concept-accuracy-improvement.pdf}}%
    \put(0.82022003,0.34949218){\makebox(0,0)[lt]{\lineheight{1.25}\smash{\begin{tabular}[t]{l}Hierarchical\end{tabular}}}}%
    \put(0.82022003,0.32703559){\makebox(0,0)[lt]{\lineheight{1.25}\smash{\begin{tabular}[t]{l}Baseline\end{tabular}}}}%
    \put(0.10201708,0.00649909){\makebox(0,0)[lt]{\lineheight{1.25}\smash{\begin{tabular}[t]{l}1:still\end{tabular}}}}%
    \put(0.21109193,0.00649909){\makebox(0,0)[lt]{\lineheight{1.25}\smash{\begin{tabular}[t]{l}2:walk\end{tabular}}}}%
    \put(0.32497887,0.00649909){\makebox(0,0)[lt]{\lineheight{1.25}\smash{\begin{tabular}[t]{l}3:run\end{tabular}}}}%
    \put(0.4372618,0.00649909){\makebox(0,0)[lt]{\lineheight{1.25}\smash{\begin{tabular}[t]{l}4:bike\end{tabular}}}}%
    \put(0.54954473,0.00649909){\makebox(0,0)[lt]{\lineheight{1.25}\smash{\begin{tabular}[t]{l}5:car\end{tabular}}}}%
    \put(0.66182766,0.00649909){\makebox(0,0)[lt]{\lineheight{1.25}\smash{\begin{tabular}[t]{l}6:bus\end{tabular}}}}%
    \put(0.77090251,0.00649909){\makebox(0,0)[lt]{\lineheight{1.25}\smash{\begin{tabular}[t]{l}7:train\end{tabular}}}}%
    \put(0.87356113,0.00649909){\makebox(0,0)[lt]{\lineheight{1.25}\smash{\begin{tabular}[t]{l}8:subway\end{tabular}}}}%
    \put(0,0){\includegraphics[width=\unitlength,page=3]{per-concept-accuracy-improvement.pdf}}%
    \put(0.0379223,0.37007823){\makebox(0,0)[lt]{\lineheight{1.25}\smash{\begin{tabular}[t]{l}85\end{tabular}}}}%
    \put(0.0379223,0.31233273){\makebox(0,0)[lt]{\lineheight{1.25}\smash{\begin{tabular}[t]{l}80\end{tabular}}}}%
    \put(0.0379223,0.25298318){\makebox(0,0)[lt]{\lineheight{1.25}\smash{\begin{tabular}[t]{l}75\end{tabular}}}}%
    \put(0.0379223,0.19363363){\makebox(0,0)[lt]{\lineheight{1.25}\smash{\begin{tabular}[t]{l}70\end{tabular}}}}%
    \put(0.0379223,0.13588814){\makebox(0,0)[lt]{\lineheight{1.25}\smash{\begin{tabular}[t]{l}65\end{tabular}}}}%
    \put(0.0379223,0.07653859){\makebox(0,0)[lt]{\lineheight{1.25}\smash{\begin{tabular}[t]{l}60\end{tabular}}}}%
  \end{picture}%
\endgroup%

%% file: img/high-affinity-score-model.pdf_tex
\begingroup%
  \makeatletter%
  \providecommand\color[2][]{%
    \errmessage{(Inkscape) Color is used for the text in Inkscape, but the package 'color.sty' is not loaded}%
    \renewcommand\color[2][]{}%
  }%
  \providecommand\transparent[1]{%
    \errmessage{(Inkscape) Transparency is used (non-zero) for the text in Inkscape, but the package 'transparent.sty' is not loaded}%
    \renewcommand\transparent[1]{}%
  }%
  \providecommand\rotatebox[2]{#2}%
  \newcommand*\fsize{\dimexpr\f@size pt\relax}%
  \newcommand*\lineheight[1]{\fontsize{\fsize}{#1\fsize}\selectfont}%
  \ifx\svgwidth\undefined%
    \setlength{\unitlength}{84.40720313bp}%
    \ifx\svgscale\undefined%
      \relax%
    \else%
      \setlength{\unitlength}{\unitlength * \real{\svgscale}}%
    \fi%
  \else%
    \setlength{\unitlength}{\svgwidth}%
  \fi%
  \global\let\svgwidth\undefined%
  \global\let\svgscale\undefined%
  \makeatother%
  \begin{picture}(1,1.17732843)%
    \lineheight{1}%
    \setlength\tabcolsep{0pt}%
    \put(0,0){\includegraphics[width=\unitlength,page=1]{high-affinity-score-model.pdf}}%
    \put(0.50035403,0.99961848){\makebox(0,0)[t]{\lineheight{1.25}\smash{\begin{tabular}[t]{c}$\mathcal{M}_{[z_i]}$\end{tabular}}}}%
    \put(0,0){\includegraphics[width=\unitlength,page=2]{high-affinity-score-model.pdf}}%
    \put(0.2337891,0.09774047){\makebox(0,0)[t]{\lineheight{1.25}\smash{\begin{tabular}[t]{c}$c_i$\end{tabular}}}}%
    \put(0,0){\includegraphics[width=\unitlength,page=3]{high-affinity-score-model.pdf}}%
    \put(0.74914796,0.09774047){\makebox(0,0)[t]{\lineheight{1.25}\smash{\begin{tabular}[t]{c}$c_j$\end{tabular}}}}%
  \end{picture}%
\endgroup%

%% file: img/linearly-separable.pdf_tex
\begingroup%
  \makeatletter%
  \providecommand\color[2][]{%
    \errmessage{(Inkscape) Color is used for the text in Inkscape, but the package 'color.sty' is not loaded}%
    \renewcommand\color[2][]{}%
  }%
  \providecommand\transparent[1]{%
    \errmessage{(Inkscape) Transparency is used (non-zero) for the text in Inkscape, but the package 'transparent.sty' is not loaded}%
    \renewcommand\transparent[1]{}%
  }%
  \providecommand\rotatebox[2]{#2}%
  \newcommand*\fsize{\dimexpr\f@size pt\relax}%
  \newcommand*\lineheight[1]{\fontsize{\fsize}{#1\fsize}\selectfont}%
  \ifx\svgwidth\undefined%
    \setlength{\unitlength}{371.82859469bp}%
    \ifx\svgscale\undefined%
      \relax%
    \else%
      \setlength{\unitlength}{\unitlength * \real{\svgscale}}%
    \fi%
  \else%
    \setlength{\unitlength}{\svgwidth}%
  \fi%
  \global\let\svgwidth\undefined%
  \global\let\svgscale\undefined%
  \makeatother%
  \begin{picture}(1,0.80780132)%
    \lineheight{1}%
    \setlength\tabcolsep{0pt}%
    \put(0,0){\includegraphics[width=\unitlength,page=1]{linearly-separable.pdf}}%
  \end{picture}%
\endgroup%

%% file: img/decision-boundaries.pdf_tex
\begingroup%
  \makeatletter%
  \providecommand\color[2][]{%
    \errmessage{(Inkscape) Color is used for the text in Inkscape, but the package 'color.sty' is not loaded}%
    \renewcommand\color[2][]{}%
  }%
  \providecommand\transparent[1]{%
    \errmessage{(Inkscape) Transparency is used (non-zero) for the text in Inkscape, but the package 'transparent.sty' is not loaded}%
    \renewcommand\transparent[1]{}%
  }%
  \providecommand\rotatebox[2]{#2}%
  \newcommand*\fsize{\dimexpr\f@size pt\relax}%
  \newcommand*\lineheight[1]{\fontsize{\fsize}{#1\fsize}\selectfont}%
  \ifx\svgwidth\undefined%
    \setlength{\unitlength}{1651.76407753bp}%
    \ifx\svgscale\undefined%
      \relax%
    \else%
      \setlength{\unitlength}{\unitlength * \real{\svgscale}}%
    \fi%
  \else%
    \setlength{\unitlength}{\svgwidth}%
  \fi%
  \global\let\svgwidth\undefined%
  \global\let\svgscale\undefined%
  \makeatother%
  \begin{picture}(1,0.41807506)%
    \lineheight{1}%
    \setlength\tabcolsep{0pt}%
    \put(0,0){\includegraphics[width=\unitlength,page=1]{decision-boundaries.pdf}}%
    \put(0.01147137,0.22965636){\rotatebox{90}{\makebox(0,0)[t]{\lineheight{1.25}\smash{\begin{tabular}[t]{c}Initial transfer-based affinity scores\end{tabular}}}}}%
    \put(0.52987399,0.00826026){\rotatebox{-0.07581957}{\makebox(0,0)[t]{\lineheight{1.25}\smash{\begin{tabular}[t]{c}Supervision budget (\# learning examples)\end{tabular}}}}}%
    \put(0.16185793,0.0230169){\makebox(0,0)[lt]{\lineheight{1.25}\smash{\begin{tabular}[t]{l}200\end{tabular}}}}%
    \put(0.40614256,0.0230169){\makebox(0,0)[lt]{\lineheight{1.25}\smash{\begin{tabular}[t]{l}300\end{tabular}}}}%
    \put(0.63771324,0.0230169){\makebox(0,0)[lt]{\lineheight{1.25}\smash{\begin{tabular}[t]{l}400\end{tabular}}}}%
    \put(0.87836515,0.0230169){\makebox(0,0)[lt]{\lineheight{1.25}\smash{\begin{tabular}[t]{l}500\end{tabular}}}}%
    \put(0.03108859,0.08567734){\makebox(0,0)[lt]{\lineheight{1.25}\smash{\begin{tabular}[t]{l}0.2\end{tabular}}}}%
    \put(0.03108859,0.1792138){\makebox(0,0)[lt]{\lineheight{1.25}\smash{\begin{tabular}[t]{l}0.4\end{tabular}}}}%
    \put(0.03108859,0.27456645){\makebox(0,0)[lt]{\lineheight{1.25}\smash{\begin{tabular}[t]{l}0.6\end{tabular}}}}%
    \put(0.03108859,0.36719473){\makebox(0,0)[lt]{\lineheight{1.25}\smash{\begin{tabular}[t]{l}0.8\end{tabular}}}}%
  \end{picture}%
\endgroup%

%% file: img/hierarchy-1.pdf_tex
\begingroup%
  \makeatletter%
  \providecommand\color[2][]{%
    \errmessage{(Inkscape) Color is used for the text in Inkscape, but the package 'color.sty' is not loaded}%
    \renewcommand\color[2][]{}%
  }%
  \providecommand\transparent[1]{%
    \errmessage{(Inkscape) Transparency is used (non-zero) for the text in Inkscape, but the package 'transparent.sty' is not loaded}%
    \renewcommand\transparent[1]{}%
  }%
  \providecommand\rotatebox[2]{#2}%
  \newcommand*\fsize{\dimexpr\f@size pt\relax}%
  \newcommand*\lineheight[1]{\fontsize{\fsize}{#1\fsize}\selectfont}%
  \ifx\svgwidth\undefined%
    \setlength{\unitlength}{264.375bp}%
    \ifx\svgscale\undefined%
      \relax%
    \else%
      \setlength{\unitlength}{\unitlength * \real{\svgscale}}%
    \fi%
  \else%
    \setlength{\unitlength}{\svgwidth}%
  \fi%
  \global\let\svgwidth\undefined%
  \global\let\svgscale\undefined%
  \makeatother%
  \begin{picture}(1,1.19858156)%
    \lineheight{1}%
    \setlength\tabcolsep{0pt}%
    \put(0,0){\includegraphics[width=\unitlength,page=1]{hierarchy-1.pdf}}%
    \put(0.1106383,0.65957447){\makebox(0,0)[t]{\lineheight{1.25}\smash{\begin{tabular}[t]{c}1\end{tabular}}}}%
    \put(0,0){\includegraphics[width=\unitlength,page=2]{hierarchy-1.pdf}}%
    \put(0.04255319,0.03120567){\makebox(0,0)[t]{\lineheight{1.25}\smash{\begin{tabular}[t]{c}2\end{tabular}}}}%
    \put(0,0){\includegraphics[width=\unitlength,page=3]{hierarchy-1.pdf}}%
    \put(0.20992908,0.03120567){\makebox(0,0)[t]{\lineheight{1.25}\smash{\begin{tabular}[t]{c}3\end{tabular}}}}%
    \put(0,0){\includegraphics[width=\unitlength,page=4]{hierarchy-1.pdf}}%
    \put(0.38014184,0.03120567){\makebox(0,0)[t]{\lineheight{1.25}\smash{\begin{tabular}[t]{c}4\end{tabular}}}}%
    \put(0,0){\includegraphics[width=\unitlength,page=5]{hierarchy-1.pdf}}%
    \put(0.52198582,0.03120567){\makebox(0,0)[t]{\lineheight{1.25}\smash{\begin{tabular}[t]{c}5\end{tabular}}}}%
    \put(0,0){\includegraphics[width=\unitlength,page=6]{hierarchy-1.pdf}}%
    \put(0.67092199,0.03120567){\makebox(0,0)[t]{\lineheight{1.25}\smash{\begin{tabular}[t]{c}6\end{tabular}}}}%
    \put(0,0){\includegraphics[width=\unitlength,page=7]{hierarchy-1.pdf}}%
    \put(0.81276596,0.03120567){\makebox(0,0)[t]{\lineheight{1.25}\smash{\begin{tabular}[t]{c}7\end{tabular}}}}%
    \put(0,0){\includegraphics[width=\unitlength,page=8]{hierarchy-1.pdf}}%
    \put(0.95460993,0.03120567){\makebox(0,0)[t]{\lineheight{1.25}\smash{\begin{tabular}[t]{c}8\end{tabular}}}}%
  \end{picture}%
\endgroup%

%% file: img/hierarchy-2.pdf_tex
\begingroup%
  \makeatletter%
  \providecommand\color[2][]{%
    \errmessage{(Inkscape) Color is used for the text in Inkscape, but the package 'color.sty' is not loaded}%
    \renewcommand\color[2][]{}%
  }%
  \providecommand\transparent[1]{%
    \errmessage{(Inkscape) Transparency is used (non-zero) for the text in Inkscape, but the package 'transparent.sty' is not loaded}%
    \renewcommand\transparent[1]{}%
  }%
  \providecommand\rotatebox[2]{#2}%
  \newcommand*\fsize{\dimexpr\f@size pt\relax}%
  \newcommand*\lineheight[1]{\fontsize{\fsize}{#1\fsize}\selectfont}%
  \ifx\svgwidth\undefined%
    \setlength{\unitlength}{264.375bp}%
    \ifx\svgscale\undefined%
      \relax%
    \else%
      \setlength{\unitlength}{\unitlength * \real{\svgscale}}%
    \fi%
  \else%
    \setlength{\unitlength}{\svgwidth}%
  \fi%
  \global\let\svgwidth\undefined%
  \global\let\svgscale\undefined%
  \makeatother%
  \begin{picture}(1,1.19858156)%
    \lineheight{1}%
    \setlength\tabcolsep{0pt}%
    \put(0,0){\includegraphics[width=\unitlength,page=1]{hierarchy-2.pdf}}%
    \put(0.1106383,0.65957447){\makebox(0,0)[t]{\lineheight{1.25}\smash{\begin{tabular}[t]{c}1\end{tabular}}}}%
    \put(0,0){\includegraphics[width=\unitlength,page=2]{hierarchy-2.pdf}}%
    \put(0.04255319,0.03120567){\makebox(0,0)[t]{\lineheight{1.25}\smash{\begin{tabular}[t]{c}2\end{tabular}}}}%
    \put(0,0){\includegraphics[width=\unitlength,page=3]{hierarchy-2.pdf}}%
    \put(0.20992908,0.03120567){\makebox(0,0)[t]{\lineheight{1.25}\smash{\begin{tabular}[t]{c}3\end{tabular}}}}%
    \put(0,0){\includegraphics[width=\unitlength,page=4]{hierarchy-2.pdf}}%
    \put(0.38014184,0.03120567){\makebox(0,0)[t]{\lineheight{1.25}\smash{\begin{tabular}[t]{c}4\end{tabular}}}}%
    \put(0,0){\includegraphics[width=\unitlength,page=5]{hierarchy-2.pdf}}%
    \put(0.52198582,0.03120567){\makebox(0,0)[t]{\lineheight{1.25}\smash{\begin{tabular}[t]{c}5\end{tabular}}}}%
    \put(0,0){\includegraphics[width=\unitlength,page=6]{hierarchy-2.pdf}}%
    \put(0.67092199,0.03120567){\makebox(0,0)[t]{\lineheight{1.25}\smash{\begin{tabular}[t]{c}6\end{tabular}}}}%
    \put(0,0){\includegraphics[width=\unitlength,page=7]{hierarchy-2.pdf}}%
    \put(0.81276596,0.03120567){\makebox(0,0)[t]{\lineheight{1.25}\smash{\begin{tabular}[t]{c}7\end{tabular}}}}%
    \put(0,0){\includegraphics[width=\unitlength,page=8]{hierarchy-2.pdf}}%
    \put(0.95460993,0.03120567){\makebox(0,0)[t]{\lineheight{1.25}\smash{\begin{tabular}[t]{c}8\end{tabular}}}}%
  \end{picture}%
\endgroup%

%% file: img/hierarchy-3.pdf_tex
\begingroup%
  \makeatletter%
  \providecommand\color[2][]{%
    \errmessage{(Inkscape) Color is used for the text in Inkscape, but the package 'color.sty' is not loaded}%
    \renewcommand\color[2][]{}%
  }%
  \providecommand\transparent[1]{%
    \errmessage{(Inkscape) Transparency is used (non-zero) for the text in Inkscape, but the package 'transparent.sty' is not loaded}%
    \renewcommand\transparent[1]{}%
  }%
  \providecommand\rotatebox[2]{#2}%
  \newcommand*\fsize{\dimexpr\f@size pt\relax}%
  \newcommand*\lineheight[1]{\fontsize{\fsize}{#1\fsize}\selectfont}%
  \ifx\svgwidth\undefined%
    \setlength{\unitlength}{254.25bp}%
    \ifx\svgscale\undefined%
      \relax%
    \else%
      \setlength{\unitlength}{\unitlength * \real{\svgscale}}%
    \fi%
  \else%
    \setlength{\unitlength}{\svgwidth}%
  \fi%
  \global\let\svgwidth\undefined%
  \global\let\svgscale\undefined%
  \makeatother%
  \begin{picture}(1,1.24631268)%
    \lineheight{1}%
    \setlength\tabcolsep{0pt}%
    \put(0,0){\includegraphics[width=\unitlength,page=1]{hierarchy-3.pdf}}%
    \put(0.07522124,0.68584071){\makebox(0,0)[t]{\lineheight{1.25}\smash{\begin{tabular}[t]{c}1\end{tabular}}}}%
    \put(0,0){\includegraphics[width=\unitlength,page=2]{hierarchy-3.pdf}}%
    \put(0.04424779,0.35693215){\makebox(0,0)[t]{\lineheight{1.25}\smash{\begin{tabular}[t]{c}2\end{tabular}}}}%
    \put(0,0){\includegraphics[width=\unitlength,page=3]{hierarchy-3.pdf}}%
    \put(0.21828909,0.35693215){\makebox(0,0)[t]{\lineheight{1.25}\smash{\begin{tabular}[t]{c}3\end{tabular}}}}%
    \put(0,0){\includegraphics[width=\unitlength,page=4]{hierarchy-3.pdf}}%
    \put(0.39528024,0.35693215){\makebox(0,0)[t]{\lineheight{1.25}\smash{\begin{tabular}[t]{c}4\end{tabular}}}}%
    \put(0,0){\includegraphics[width=\unitlength,page=5]{hierarchy-3.pdf}}%
    \put(0.50294985,0.03244838){\makebox(0,0)[t]{\lineheight{1.25}\smash{\begin{tabular}[t]{c}5\end{tabular}}}}%
    \put(0,0){\includegraphics[width=\unitlength,page=6]{hierarchy-3.pdf}}%
    \put(0.65781711,0.03244838){\makebox(0,0)[t]{\lineheight{1.25}\smash{\begin{tabular}[t]{c}6\end{tabular}}}}%
    \put(0,0){\includegraphics[width=\unitlength,page=7]{hierarchy-3.pdf}}%
    \put(0.80530973,0.03244838){\makebox(0,0)[t]{\lineheight{1.25}\smash{\begin{tabular}[t]{c}7\end{tabular}}}}%
    \put(0,0){\includegraphics[width=\unitlength,page=8]{hierarchy-3.pdf}}%
    \put(0.95280236,0.03244838){\makebox(0,0)[t]{\lineheight{1.25}\smash{\begin{tabular}[t]{c}8\end{tabular}}}}%
  \end{picture}%
\endgroup%

%% file: img/hierarchy-4.pdf_tex
\begingroup%
  \makeatletter%
  \providecommand\color[2][]{%
    \errmessage{(Inkscape) Color is used for the text in Inkscape, but the package 'color.sty' is not loaded}%
    \renewcommand\color[2][]{}%
  }%
  \providecommand\transparent[1]{%
    \errmessage{(Inkscape) Transparency is used (non-zero) for the text in Inkscape, but the package 'transparent.sty' is not loaded}%
    \renewcommand\transparent[1]{}%
  }%
  \providecommand\rotatebox[2]{#2}%
  \newcommand*\fsize{\dimexpr\f@size pt\relax}%
  \newcommand*\lineheight[1]{\fontsize{\fsize}{#1\fsize}\selectfont}%
  \ifx\svgwidth\undefined%
    \setlength{\unitlength}{278.25bp}%
    \ifx\svgscale\undefined%
      \relax%
    \else%
      \setlength{\unitlength}{\unitlength * \real{\svgscale}}%
    \fi%
  \else%
    \setlength{\unitlength}{\svgwidth}%
  \fi%
  \global\let\svgwidth\undefined%
  \global\let\svgscale\undefined%
  \makeatother%
  \begin{picture}(1,1.13881402)%
    \lineheight{1}%
    \setlength\tabcolsep{0pt}%
    \put(0,0){\includegraphics[width=\unitlength,page=1]{hierarchy-4.pdf}}%
    \put(0.04043127,0.51886792){\makebox(0,0)[t]{\lineheight{1.25}\smash{\begin{tabular}[t]{c}1\end{tabular}}}}%
    \put(0,0){\includegraphics[width=\unitlength,page=2]{hierarchy-4.pdf}}%
    \put(0.0902965,0.0296496){\makebox(0,0)[t]{\lineheight{1.25}\smash{\begin{tabular}[t]{c}2\end{tabular}}}}%
    \put(0,0){\includegraphics[width=\unitlength,page=3]{hierarchy-4.pdf}}%
    \put(0.24932615,0.0296496){\makebox(0,0)[t]{\lineheight{1.25}\smash{\begin{tabular}[t]{c}3\end{tabular}}}}%
    \put(0,0){\includegraphics[width=\unitlength,page=4]{hierarchy-4.pdf}}%
    \put(0.41105121,0.0296496){\makebox(0,0)[t]{\lineheight{1.25}\smash{\begin{tabular}[t]{c}4\end{tabular}}}}%
    \put(0,0){\includegraphics[width=\unitlength,page=5]{hierarchy-4.pdf}}%
    \put(0.5458221,0.0296496){\makebox(0,0)[t]{\lineheight{1.25}\smash{\begin{tabular}[t]{c}5\end{tabular}}}}%
    \put(0,0){\includegraphics[width=\unitlength,page=6]{hierarchy-4.pdf}}%
    \put(0.68733154,0.0296496){\makebox(0,0)[t]{\lineheight{1.25}\smash{\begin{tabular}[t]{c}6\end{tabular}}}}%
    \put(0,0){\includegraphics[width=\unitlength,page=7]{hierarchy-4.pdf}}%
    \put(0.82210243,0.0296496){\makebox(0,0)[t]{\lineheight{1.25}\smash{\begin{tabular}[t]{c}7\end{tabular}}}}%
    \put(0,0){\includegraphics[width=\unitlength,page=8]{hierarchy-4.pdf}}%
    \put(0.95687332,0.0296496){\makebox(0,0)[t]{\lineheight{1.25}\smash{\begin{tabular}[t]{c}8\end{tabular}}}}%
  \end{picture}%
\endgroup%